\documentclass{article}

\usepackage{arxiv}

\usepackage{graphicx}
\usepackage{hyperref}
\usepackage{color}
\usepackage{subfiles}
\usepackage{balance} 

\usepackage{amsmath,amsfonts,bm}

\def\eqref#1{equation~\ref{#1}}

\def\1{\bm{1}}

\DeclareMathAlphabet{\mathsfit}{\encodingdefault}{\sfdefault}{m}{sl}
\SetMathAlphabet{\mathsfit}{bold}{\encodingdefault}{\sfdefault}{bx}{n}

\def\gA{{\mathcal{A}}}

\def\gD{{\mathcal{D}}}

\def\gL{{\mathcal{L}}}

\def\gO{{\mathcal{O}}}
\def\gP{{\mathcal{P}}}

\def\gS{{\mathcal{S}}}

\def\sI{{\mathbb{I}}}

\def\sV{{\mathbb{V}}}

\newcommand{\E}{\mathbb{E}}

\newcommand{\KL}{D_{\mathrm{KL}}}
\newcommand{\TV}{D_{\mathrm{TV}}}

\DeclareMathOperator*{\argmin}{arg\,min}

\DeclareMathOperator{\Tr}{Tr}

\usepackage{amsthm, amssymb}
\theoremstyle{plain}
\newtheorem{theorem}{Theorem}[section]

\newtheorem{lemma}[theorem]{Lemma}
\newtheorem{corollary}[theorem]{Corollary}
\theoremstyle{definition}
\newtheorem{definition}[theorem]{Definition}

\theoremstyle{remark}
\usepackage{subfigure}

\usepackage{array}
\usepackage{algorithm}
\usepackage{algorithmic}
\usepackage{wrapfig}
\usepackage{capt-of}
\usepackage{booktabs}
\usepackage{varwidth}
\usepackage{lineno}
\usepackage{multirow}
\usepackage{enumitem}
\usepackage{url}
\usepackage[capitalize,noabbrev]{cleveref}

\newcommand{\methodname}{AdapMen}
\newcommand{\hatpi}{\hat{\pi}}

\hypersetup{
    colorlinks=true,
    linkcolor=blue,
    filecolor=blue,
    urlcolor=blue,
    citecolor=cyan,
}

\usepackage{natbib}
\newcommand{\citeN}{\cite}

\title{How To Guide Your Learner: Imitation Learning with Active Adaptive Expert Involvement}

\author{Xu-Hui Liu$^{1,}$\thanks{Equal Contribution}~~, Feng Xu$^{1,}$\footnotemark[1]~~, Xinyu Zhang$^{1}$, Tianyuan Liu$^{1}$, Shengyi Jiang$^{2}$, Ruifeng Chen$^{1}$, \\ \bf Zongzhang Zhang$^{1}$, Yang Yu$^{1,3,}$\thanks{Corresponding Author} \\
	$^1$National Key Laboratory for Novel Software Technology, Nanjing University, Nanjing, China \\
	$^2$The University of Hong Kong, Hong Kong, China\\
	$^3$Peng Cheng Laboratory, Shenzhen, China\\
  \{liuxh, xufeng, zhangxinyu\}@lamda.nju.edu.cn, 191300034@smail.nju.edu.cn, syjiang@cs.hku.hk,\\
  chenrf@lamda.nju.edu.cn, \{zzzhang, yuy\}@nju.edu.cn
}

\date{}

\begin{document}

\maketitle

\begin{abstract}
Imitation learning aims to mimic the behavior of experts without explicit reward signals. Passive imitation learning methods which use static expert datasets typically suffer from compounding error, low sample efficiency, and high hyper-parameter sensitivity. In contrast, active imitation learning methods solicit expert interventions to address the limitations. However, recent active imitation learning methods are designed based on human intuitions or empirical experience without theoretical guarantee.
In this paper, we propose a novel active imitation learning framework based on a teacher-student interaction model, in which the teacher's goal is to identify the best teaching behavior and actively affect the student's learning process. By solving the optimization objective of this framework, we propose a practical implementation, naming it \methodname.
Theoretical analysis shows that \methodname\ can improve the error bound and avoid compounding error under mild conditions. Experiments on the MetaDrive benchmark and Atari 2600 games validate our theoretical analysis and show that our method achieves near-expert performance with much less expert involvement and total sampling steps than previous methods. The code is available at \href{https://github.com/liuxhym/AdapMen}{\textcolor{blue}{https://github.com/liuxhym/AdapMen}}.
\end{abstract}

\section{Introduction}\label{sec_intro}

Imitation Learning (IL)~\citep{bc, dagger, gail} aims to learn a policy from expert demonstrations with no explicit task-relevant knowledge like reward and transition.
IL has achieved huge success in a variety of domains, including games~\citep{dagger, alphago} and recommendation systems~\citep{virtual_taobao, recommendation}.

The traditional IL method Behavior Cloning (BC)~\citep{bc} imitates expert behaviors via supervised learning. Although BC works fine in simple environments, it requires a lot of data and small errors compound quickly when the learned policy deviates from the states in the expert dataset. This issue can be formalized by the sub-optimality bound of the learned policy, which is $\tilde{\gO}(\epsilon_bH^2)$ for BC~\citep{bc}, where $\epsilon_b$ is the optimization error, $H$ is the horizon of the Markov Decision Processes (MDPs) and $\tilde{\gO}$ means the constant and log terms are omitted. The quadratic dependency on $H$ is known as the \textit{compounding error} issue. 

To tackle the compounding error issue, Apprenticeship Learning (AL)~\citep{apprenticeship, fem} and Adversarial Imitation Learning (AIL)~\citep{gail, dac, ValueDICE, iq_learn} algorithms introduce interactions with environment. They first infer a reward function from expert demonstrations, then learn a corresponding policy by Reinforcement Learning (RL). The sub-optimality bound is then reduced to $\tilde{\gO}(\epsilon_gH)$ \cite{xu2020}, where $\epsilon_g$ is the optimization error of AL and AIL.  From another perspective, DAgger~\citep{dagger} attributes the compounding error issue to the difference between the train distribution and test distribution. Thus, DAgger queries the expert for action labels corresponding to each state visited by the learner~\citep{dagger}.

Despite the reduction of the order of $H$, the complicated optimization process of AL and AIL leads to even worse sample complexity than BC~\citep{ail_finite_sample}.  Additionally, these algorithms are highly sensitive to hyper-parameters and are hard to converge in practice~\citep{zhang2022}.
DAgger also relies on an additional assumption that the learner can recover from mistakes made by itself to a certain extent, which is known as the $\mu$-recoverability condition on the MDPs. \citeN{value_interaction} proves that DAgger has a better theoretical guarantee than BC under such an assumption while \cite{limits} shows a negative result for general cases.  Moreover, the assumption is satisfied only when any sub-optimal action leads to little performance degradation, which can be impractical, e.g., in risk-sensitive environments~\citep{limits, value_interaction}. Some recent methods~\citep{HG-dagger, ensembledagger, safetydagger, thriftydagger, HACO} modify DAgger so that they only solicit expert interventions based on certain criteria. Though these methods achieve certain empirical success, there were no theoretical understanding of these methods and their design of intervention criteria is totally intuitive, hindering further algorithmic design.

To address the issues in previous methods, we study the IL problems from a new perspective. From experience, sometimes experts are not the best teachers.
For example, many legendary players end up with controversial coaching careers. Experts advance disciplines, while teachers advance learners. Inspired by the idea of machine teaching~\citep{mt, qian2022}
, we formulate the IL process as a teacher-student framework. In this framework, the teacher decides what to teach and how to impart knowledge rather than simply correcting the student. With more attentive help from the teacher, the student agent can learn faster.

We formalize this intuition by introducing an optimization problem on minimizing the value loss of the learned policy. By solving the optimization problem in the framework, we obtain a novel imitation learning method \textbf{A}ctive a\textbf{da}ptive ex\textbf{p}ert involve\textbf{Men}t (\methodname), where a teacher actively involves in the learner's interaction with the environment and adjusts its teaching behavior accordingly. 
The overall interaction structure is illustrated in Fig.~\ref{fig:framework}, where the criterion and the expert is together viewed as the teacher. 
At each time step, a criterion calculated from expert actions judges whether to take the learner's action or ask the expert to take over control.

The sub-optimality and sample complexity bounds of \methodname\  and other typical IL methods are listed in Tab.~\ref{tab:complexity}. Under mild conditions, \methodname\ achieves no compounding error with much lower sample complexity than previous methods. To validate our theories, we also experimentally verify the validity of the assumption and demonstrate the power of \methodname\  in several tasks.

\begin{figure}[tbp]
  \centering
  \centering
    \includegraphics[width=0.7\linewidth, trim=25 20 50 0]{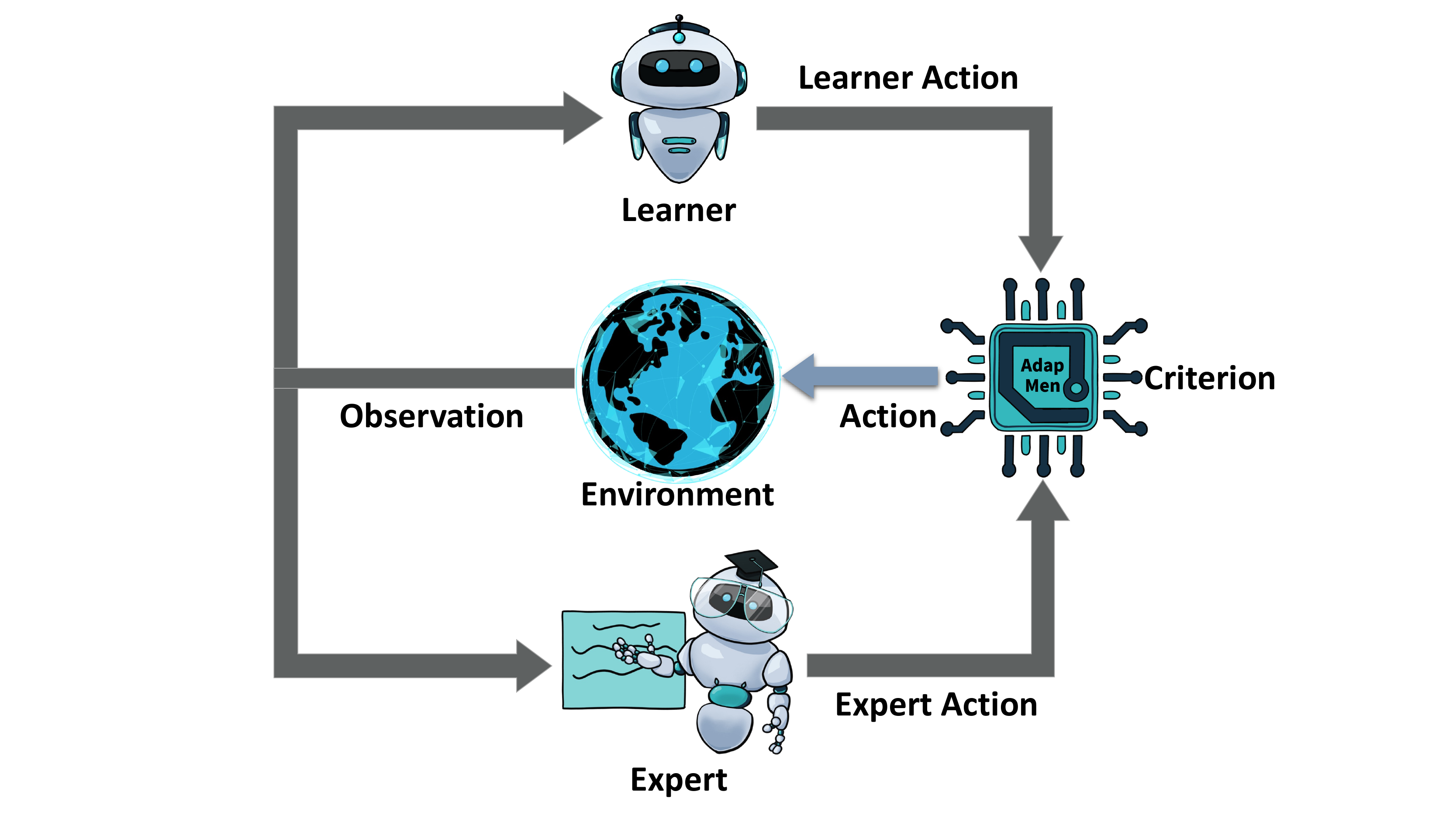}
    \captionof{figure}{Framework of \methodname}
    \label{fig:framework}
\end{figure}
\begin{table}[tbp]
\vspace{4mm}
    \centering
     \begin{tabular}[b]{|l|l|l|}
    \hline
            &   \parbox{1.5cm}{Sub-\\optimality}       &   \parbox{1.5cm}{Sample\\Complexity}  \\ \hline
    BC      &   $\tilde{\gO}(\epsilon_bH^2)$      &   $\tilde{\gO}(\frac{|\gS|H^2}{\epsilon})$    \\ \hline
    AIL     &   $\tilde{\gO}(\epsilon_gH)$        &   $\tilde{\gO}(\frac{|\gS|H^2}{\epsilon^2})$  \\ \hline
    DAgger  &   $\tilde{\gO}(\mu \epsilon_bH)$   &   $\tilde{\gO}(\frac{\mu|\gS|H}{\epsilon})$   \\ \hline
\methodname &   $\tilde{\gO}(\epsilon_bH)$        &   $\tilde{\gO}(\frac{|\gS|H}{\epsilon})$      \\ \hline
    \end{tabular}
    \vspace{2.5mm}
    \caption{Theoretical Guarantee of IL Methods. $\tilde{\gO}$ means the log term of $H$ is omitted.}
    \label{tab:complexity}
\end{table}

\section{Related Work}

\textbf{Imitation Learning.}
The most traditional approach to imitation learning is Behavioral Cloning (BC)~\citep{BC1, BC2, bc}, where a classifier or regressor is trained to fit the behaviors of the expert. This simple form of IL suffers from high compounding error because of covariate shifts. By allowing the learner agent to further interact with the environment, Apprenticeship Learning (AL)~\citep{apprenticeship, fem} infers a reward function from expert demonstrations by Inverse Reinforcement Learning (IRL)~\citep{irl} and learns a policy with Reinforcement Learning (RL) using the recovered reward function. In this way, the learner can correct its behavior on unseen states to mitigate the compounding error issue. Recently, based on Generative Adversarial Network (GAN)~\citep{gan}, Adversarial Imitation Learning (AIL)~\citep{gail, dac, ValueDICE} performs state-action distribution matching in an adversarial manner and has a stronger empirical performance than AL. Since AL and AIL have access to environment transitions, they are classified as known-transition methods. Notwithstanding the compounding error issue, this type of method is highly sensitive to hyper-parameters and hard to converge in practice~\citep{zhang2022}. Different from the known-transition methods, DAgger-style algorithms~\citep{dagger, aggrevate, aggrevated} address the covariate shift by querying the expert online. Without the min-max optimization in known-transition methods, DAgger-style algorithms tend to be more stable. However, these algorithms can only avoid compounding error under the $\mu$-recoverability assumption, which is often not satisfied in risk-sensitive environments~\citep{limits, value_interaction}. Our method  \methodname\ is free from the $\mu$-recoverability assumption and the hyper-parameters are automatically tuned.

\textbf{Human-in-the-loop.} Many works focus on incorporating human interventions in the training loop of RL or IL paradigms. DAgger~\citep{dagger} can be seen as one of the human-in-the-loop methods if the expert is a human. DAgger requires experts to provide action labels without being fully in control of the system, which can introduce safety concerns and is very likely to degrade the quality of the collected labels due to the loss of direct feedback. To address this challenge, a list of learning from intervention approaches have been proposed to empower humans to intervene and guide the learner agent to safe states. "Human-Gated" approaches~\citep{HG-dagger, EIL, HACO} require humans to determine when the agent needs help and when to cede control, which is unreliable because of the high randomness of human behavior. In contrast, ``Agent-Gated'' approaches~\citep{safetydagger, ensembledagger, lazydagger, thriftydagger} allow the learner agent to actively seek human interventions based on certain criteria including the novelty or the risk of the visited states. However, all of the criteria are heuristic without theoretical guarantees and the hyper-parameters are hard to tune. Our method \methodname\ can actively involve in the interaction process and adaptively adjust its intervention probability.

\section{Background}
\label{sec:background}
Consider an MDP task denoted by $M = (\gS,\gA,\gP,H,r,\rho)$, where $\gS$ is the state space, $\gA$ is the action space, $\gP:\gS\times \gA \rightarrow \gS$ is the transition function, $H$ is the planning horizon, $r: \gS\times \gA \rightarrow \mathbb{R}$ is the reward function, and $\rho$ is the distribution of initial states. Without loss of generality, we assume $r(s,a)\in[0,1]$.  A policy is defined as $\pi(\cdot\mid s)$, which outputs an action distribution. To facilitate later analysis, we introduce the state-action distribution at time step $t$ as follows: 
$$\begin{aligned}
d_h^\pi(&s,a)=\textnormal{Pr}\left(s_h=s,a_h=a|s_1\sim \rho, a_t\sim \pi(s_t), s_{t+1}\sim \gP(\cdot|s_t,a_t), t\in[h]\right),
\end{aligned}$$
where $[h]=\{1,2,\dots,h\}$. We define 
$$d^\pi=\frac{1}{H}\sum_{h=1}^H d^\pi_h,$$
which is the average distribution of states if we follow policy $\pi$ for $H$ steps.

In imitation learning, the reward function of a task is not accessible. Instead, the learner agent has access to an expert with policy $\pi^*$, and the goal is to recover the policy $\pi^*$ by learning from labeled training data, e.g., state-action pairs generated by an expert agent. Following \cite{xu2020} and \cite{limits}, we assume the expert policy is deterministic in the theoretical analysis, while it can be stochastic in practice.

To measure the quality of a learner policy, we define the \textit{policy value} as
$$
\begin{aligned}
J(\pi)=\E\Bigg[\sum_{h=1}^H&r(s_h,a_h)|s_1\sim \rho; a_h\sim \pi(\cdot|s_h),s_{h+1}\sim \gP(\cdot|s_h,a_h), \forall h\in [H]\Bigg].
\end{aligned}$$ 
This is the cumulative return for the learner agent in the task demonstrated by the expert. Accordingly, the quality of imitation learning is measured by the \textit{sub-optimality gap}: $J(\pi^*)-J(\pi)$. We also introduce the Q-function at time step $h$:
$$
\begin{aligned}
Q^\pi_h(s,a)=\E\Bigg[&\sum_{t=h}^Hr(s_t,a_t)|s_h=s, a_h=a;
        a_t\sim \pi(\cdot|s_t),s_{t+1}\sim \gP(\cdot|s_t,a_t), \forall t\in \{h+1, \dots, H\}\Bigg].
\end{aligned}
$$
For brevity, we use $Q^{\pi_1}_h(s,\pi_2)$ as a shorthand of $\E_{a\sim \pi_2}Q^{\pi_1}_h(s,a)$. Then, $J(\pi)$ and $J(\pi^*)$ can be denoted as:
\begin{align*}
    J(\pi)=\E_{s\sim \rho}Q^\pi_H(s,\pi),\quad  J(\pi^*)=\E_{s\sim \rho}Q^{\pi^*}_H(s,\pi^*).
\end{align*}

\section{Teacher-Student Interaction Model}\label{sec_method}

Given the inspiration that experts may not be the best teachers, we construct a teaching policy for the agent. In the learning process, the agent aims to mimic the teacher policy instead of the expert policy.
This intuition can be formulated as the following optimization problem:
\begin{equation}\begin{aligned}
    \min_{\pi'} \quad J(\pi^*)-J(\pi_{\pi'})\quad \text{s.t.} \ \ \E_{s\sim \beta}\ell(s,\pi_{\pi'},\pi')\leq \epsilon_b, 
\end{aligned}
\label{eq_constraint}
\end{equation}
where $\pi'$ is the teaching policy, $\pi_{\pi'}$ is the corresponding learned student policy,  $\beta$ is the data distribution of the buffer that stores intervened samples, $\ell(s,\pi_{\pi'},\pi')$ is the 0-1 loss, i.e., $\ell(s,\pi_{\pi'},\pi')=0$ if $\pi_{\pi'}(\cdot|s)=\pi'(\cdot|s)$ and $\ell(s,\pi_{\pi'},\pi')=1$ otherwise, and $\epsilon_b$ is the upper bound of the optimization loss. Intuitively, we aim to find a policy $\pi'$ that generates data to not only correct the learner when it deviates from the desired behavior, but also helps it learn as quickly as possible.

Denote $\pi$ as the policy before the policy optimization process, i.e., $\pi_{\pi'}$ is optimized from $\pi$. Because it is useless to store the data coinciding with the agent policy, a natural choice for the distribution of buffer is
\begin{equation}\label{eq_beta}
 \beta(s)=\frac{1}{H\delta}\sum_{h=1}^H\mathbb I(\pi(\cdot|s)\neq \pi'(\cdot|s))d_h^{\pi'}(s).
\end{equation}
That is, we only save the samples when $\pi$ and $\pi'$ behave differently. $\delta$ is the normalization factor for the distribution of the buffer, i.e.,
\begin{equation}\label{eq_delta}
\delta=\sum_s\frac{1}{H}\sum_{h=1}^H\mathbb I(\pi'(\cdot|s)\neq \pi(\cdot|s))d^{\pi'}_h(s)=\mathbb E_{s\sim d^{\pi'}}\mathbb I(\pi'(\cdot|s)\neq \pi(\cdot|s)).
\end{equation}

Before solving this optimization problem, we introduce Lemma~\ref{lem:pdlema} for better understanding of the derivation.
\begin{lemma}[Policy Difference Lemma~\citep{kakade}]\label{lem_policy}
For any policies $\pi_1$ and $\pi_2$,
$$J(\pi_1)-J(\pi_2)=\sum_{h=1}^H\E_{s\sim d^{\pi_1}_h}[Q^{\pi_2}_h(s,\pi_1)-Q^{\pi_2}_h(s,\pi_2)].$$
\label{lem:pdlema}
\end{lemma}
With this lemma, we rewrite the optimization objective as
\begin{equation}
    \begin{aligned}
    &\quad \ J(\pi^*)-J(\pi_{\pi'})=J(\pi^*)-J(\pi')+J(\pi')-J(\pi_{\pi'})\\
    &\overset{(a)}{=}\sum_{h=1}^H\mathbb E_{s\sim d^{\pi'}_h}[Q_{h}^{\pi^*}(s,\pi^*)-Q_{h}^{\pi^*}(s,\pi')]\\
    &\qquad +\sum_{h=1}^H\mathbb E_{s\sim d^{\pi'}_h}[Q_{h}^{\pi_{\pi'}}(s,\pi')-Q_{h}^{\pi_{\pi'}}(s,\pi_{\pi'})]. 
\end{aligned}
\label{eq_double_q}
\end{equation}

(a) is derived from Lemma~\ref{lem_policy}. The minimization of the first term implies the teaching policy $\pi'$ should be similar to the expert policy $\pi^*$, while the minimization of the second term implies $\pi'$ should be close to $\pi_{\pi'}$. Note that we cannot determine $\pi'$ simply from $\pi_{\pi'}$ since $\pi_{\pi'}$ is learned from $\pi'$. However, $\pi'$ can be close to $\pi_{\pi'}$ if we assume $\pi_{\pi'}(\cdot|s)=\pi(\cdot|s)$ if $\pi(\cdot|s)=\pi'(\cdot|s)$. The assumption is 
straightforward because $\pi(\cdot|s)=\pi'(\cdot|s)$ implies we do not need to do optimization on state $s$, thus $\pi_{\pi'}$ stays unchanged on this state. Therefore, the overall optimization leads to a trade-off of $\pi'$ between $\pi^*$ and $\pi$.

To decompose the objective into a more tractable one, we assume $Q^\pi_h$ can be upper-bounded by $\Delta$, then
\begin{equation}\label{eq_Q_relax}
\begin{aligned}
&\quad \ \sum_{h=1}^H\mathbb E_{s\sim d^{\pi'}_h}[Q_{h}^{\pi_{\pi'}}(s,\pi')-Q_{h}^{\pi_{\pi'}}(s,\pi_{\pi'})]\\
&\leq \Delta\sum_{h=1}^H\mathbb E_{s\sim d^{\pi'}_h}\sI(\pi'(\cdot|s)\neq \pi_{\pi'}(\cdot|s)).
\end{aligned}
\end{equation}

In this way, the problem is transformed to increasing the probability that $\pi'$ equals $\pi_{\pi'}$ and reducing the value degradation between $\pi'$ and $\pi^*$ simultaneously. 

Applying the constraint in (\ref{eq_constraint}) to the right-hand side of Eq.~(\ref{eq_Q_relax}) with the mentioned $\Delta$, we have
\begin{align}
&\quad \ \Delta\sum_{h=1}^H\mathbb E_{s\sim d^{\pi'}_h}\sI(\pi'(\cdot|s)\neq \pi_{\pi'}(\cdot|s))\\
&\overset{(b)}{\leq}\Delta\sum_{h=1}^H\mathbb E_{s\sim d^{\pi'}_h}\sI(\pi'(\cdot|s)\neq \pi(\cdot|s))\sI(\pi'(\cdot|s)\neq \pi_{\pi'}(\cdot|s))\\
&\overset{(c)}{=}\Delta H\delta \ \mathbb E_{s\sim \beta}\sI(\pi'(\cdot|s)\neq \pi_{\pi'}(\cdot|s))\overset{(d)}{\leq} \Delta H\delta \epsilon_b\\
&\overset{(e)}{=}\Delta\epsilon_b\sum_{h=1}^H\E_{s\sim d_h^{\pi'}}\sI(\pi'(\cdot|s)\neq \pi(\cdot|s)), \label{eq_pi_part}
\end{align}
where (b) uses the fact that $\pi'(\cdot|s)\neq\pi_{\pi'}(\cdot|s)$ implies $\pi'(\cdot|s)=\pi_{\pi'}(\cdot|s)$, (c) is derived from the definition of $\beta$ in Eq.~(\ref{eq_beta}), (d) uses the condition $\E_{s\sim \beta}\ell(s,\pi_{\pi'},\pi')\leq \epsilon_b$, and (e) is derived from the definition of $\delta$ in Eq.~(\ref{eq_delta}).

The first term of Eq.~(\ref{eq_double_q}) can be rewritten as 

\begin{align}&\sum_{h=1}^H\mathbb E_{s\sim d^{\pi'}_h}[Q_{h}^{\pi^*}(s,\pi')-Q_{h}^{\pi^*}(s,\pi^*)]\\=&\sum_{h=1}^H\mathbb E_{s\sim d^{\pi'}_h}[Q_{h}^{\pi^*}(s,\pi')-Q_{h}^{\pi^*}(s,\pi^*)]\sI(\pi'(\cdot|s)\neq \pi^*(\cdot|s)). \label{eq_pi*_part}
\end{align}

The added $\sI(\pi'(\cdot|s)\neq \pi^*(\cdot|s))$ does not contribute to this term, because $Q_{h}^{\pi^*}(s,\pi')-Q_{h}^{\pi^*}(s,\pi^*)=0$ when $\pi'(\cdot|s)= \pi^*(\cdot|s)$. The total value loss is composed of Eq.~(\ref{eq_pi_part}) and Eq.~(\ref{eq_pi*_part}). Fixing the distribution $d^{\pi'}$, Eq.~(\ref{eq_pi_part}) equals 0 if $\pi'(\cdot|s)=\pi(\cdot|s)$ and Eq.~(\ref{eq_pi*_part}) equals 0 if $\pi'(\cdot|s)=\pi^*(\cdot|s)$. Thus, the agent will suffer from a $Q_h^{\pi^*}(s,\pi^*)-Q^{\pi^*}_h(s,\pi')$ value loss if $\pi'(\cdot|s)=\pi(\cdot|s)$, and suffer from a $\Delta\epsilon_b$ value loss if $\pi'(\cdot|s)=\pi^*(\cdot|s)$. In this way, proper choice of $\pi'$ is
\begin{equation}\label{eq_pi_prime}
    \pi'(\cdot|s)=\left\{
    \begin{aligned}
    & \pi^*(\cdot|s) \quad \textnormal{if} \,\,\,  Q_{h}^{\pi^*}(s,\pi^*)-Q_{h}^{\pi^*}(s,\pi)\geq \Delta\epsilon_b,\\
    & \pi(\cdot|s) \quad\,\; \textnormal{otherwise}.
    \end{aligned}
    \right.
\end{equation}

The resultant $\pi'$ switches between the expert policy and the learner policy according to whether $Q_{h}^{\pi^*}(s,\pi^*)-Q_{h}^{\pi^*}(s,\pi)$ exceeds the threshold. In other words, the expert intervenes the interaction when deemed necessary according to the $Q$-value difference. 

In the teacher-student interaction model, the intervention mode of the teacher is somewhat similar to DAgger-based active learning methods~\citep{HG-dagger, ensembledagger, safetydagger,thriftydagger}. The good performance achieved by them can be explained in the way that their intervention strategies make the expert a better teacher. 

Note that Eq.~(\ref{eq_pi_prime}) does not tell us how to design $\pi'$ as $\Delta$ is not available. However, it exposes the mode of a good teacher: let the expert intervenes in the interaction according to the value of $Q_{h}^{\pi^*}(s,\pi^*)-Q_{h}^{\pi^*}(s,\pi)$ and a threshold. Denote the threshold as $p$, the remaining work is to analyze the influence of $p$ and figure out a proper $p$.

\section{Analysis}\label{sec_analysis}
In this section, we analyze the theoretical properties of the intervention mode in both infinite and finite sample cases, and compare it with previous IL approaches.

First, we derive the sub-optimality bound for the teacher-student interaction model in the infinite sample case. The result is shown in the following theorem, whose proof can be found in Appendix~\ref{sec_proof}.
\begin{theorem}\label{thm_infinite}
Let $\pi$ be a policy such that $\mathbb E_{s\sim \beta}[\ell(s, \pi, \pi')]\leq \epsilon_b$, then $J(\pi^*)-J(\pi) \leq pH+\delta \epsilon_b  H^2$, where $\delta=\mathbb E_{s\sim d^{\pi'}}\mathbb I(Q_h^{\pi^*}(s,\pi^*)-Q_h^{\pi^*}(s,\pi)>p)$.
\end{theorem}

\noindent\textbf{Remark 1. }It seems $\gO(H^2)$, the term in the BC method, also appears in this sub-optimality bound. However, $\delta$ can be small if $p$ is properly chosen and may even nullify the effect of $\gO(H^2)$. The definition of $\delta$ implies it decreases as $p$ increases, while the first term $pH$ increases as $p$ increases. Therefore, $p$ provides a trade-off between the two terms. Intuitively, the first term is the error induced by neglecting some erroneous actions, while the second term is caused by optimization error.

\noindent\textbf{Remark 2. }BC is a special case of our method. When $p$ equals 0, $\delta$ equals 1. In this case, the expert takes over the entire training process, which is exactly the paradigm of BC. Replacing $p$ with 0 and  $\delta$  with 1, the bound becomes $\epsilon_b H^2$, which is the sub-optimality bound of BC, as shown in Appendix~\ref{sec_review}. Therefore, BC is the upper bound of sub-optimality in our framework.

\noindent\textbf{Remark 3. }Suppose $Q_{h}^{\pi^*}(s,\pi^*)-Q_{h}^{\pi^*}(s,\pi)$ follows a distribution $P$, then $p$ equals the $\delta$ quantile of $P$. If $P$ is concentrated, in other words, $P$ has strong tail decay, then a little increase in $p$ leads to a large drop of $\delta$, and the error bound can be improved to a great extent.

When $P$ belongs to the Sub-Exponential distribution class, which includes many common distributions, e.g., Gaussian distribution, exponential distribution and Bernoulli distribution, we have
\begin{corollary}\label{corollary}
If distribution $P$ belongs to $\gO(\epsilon_b)$-Sub-Exponential distribution class with expectation $\gO(\epsilon_b)$, let $p=\Omega(\epsilon_b\log H)$, then $J(\pi^*)-J(\pi)=\tilde{\gO}(\epsilon_bH)$, where $\tilde{\gO}$ omits the constant and $\log$ term.
\end{corollary}
The proof is given in Appendix~\ref{sec_proof}. For brevity, we use $D_Q$ to denote $Q^{\pi^*}_h(s,\pi^*)-Q^{\pi^*}_h(s,\pi)$ for the remaining of this paper. This corollary implies our method can avoid compounding error under a mild assumption on the distribution of $D_Q$. 
In Sec.~\ref{sec_experiment}, we show that the distribution $P$ in actual tasks satisfies this assumption.

We then derive the sub-optimality bound in the finite sample case. Let $\{\hatpi_i\}_{i=1}^N$ be the sequence of policies generated by our method in $N$ iterations with a fixed $p$, and $\delta_i=\E_{s\sim d^{\hatpi'_i}}\sI(Q^{\pi^*}_h(s,\pi^*)-Q^{\pi^*}_h(s,\hatpi_i)>p)$, then we obtain the following theorem.
\begin{theorem}\label{thm_finite}
Let $\hatpi=\frac{1}{N}\sum_i\hatpi_i $, 
then $J(\pi^*)-\E[J(\hatpi)]
\lesssim 
pH+\delta
\frac{|\gS|H^2}{N}$, where $\delta=\frac{1}{N}\sum_i\delta_i$ and $\lesssim$ omits the constant and the log term. 

If the condition of Corollary~\ref{corollary} is satisfied for all $N$ iterations, the bound can be improved as $J(\pi^*)-\E[J(\hatpi)]\lesssim \frac{|\gS|H}{N}$.
\end{theorem}

The bound of sample complexity can be derived from this theorem. Let value loss be $\epsilon$, then $N=\tilde{\gO}(\frac{\delta|\gS|H^2}{\epsilon-pH})$. Under the condition of Corollary~\ref{corollary}, the sample complexity is $\tilde{\gO}(\frac{|\gS|H}{\epsilon})$. This shows our method can also avoid the quadratic term of $H$ in the sample complexity. In contrast, AL and AIL methods suffer from such a term in complexity even if the compounding error in the sub-optimality bound is avoided.

\section{Practical Implementation}\label{sec_practical}

In this section, we design a practical algorithm based on the analysis in Sec.~\ref{sec_method} and \ref{sec_analysis}. The key idea is to find a proper value of the threshold $p$ and a  surrogate of $Q_h^{\pi^*}$ when $Q_h^{\pi^*}$ is not available.

To facilitate our derivation, we first introduce the definition of TV divergence and KL divergence. 
\begin{definition}
Let $P$ and $Q$ be two distributions over a sample space $\gS$ , then the TV divergence between $P$ and $Q$, $\TV(P,Q)$, is defined as
$$\TV(P,Q)=\frac{1}{2}\int |P(s)-Q(s)|ds.$$

The KL divergence between $P$ and $Q$, $\KL(P,Q)$, is defined as
$$\KL(P,Q)=\int P(s)\log\frac{P(s)}{Q(s)}ds.$$
\end{definition}

\subsection{The choice of $p$}
According to Corollary~\ref{corollary}, the sub-optimality bound is small when  the assumption on $P$, i.e., $p=\Omega(\epsilon_bH)$, is satisfied. However, letting $p=\Omega(\epsilon_bH)$ is inappropriate because it cannot generalize to other distribution classes and the constant in $\Omega$ is difficult to determine. 

To avoid the drawbacks of Corollary~\ref{corollary}, we choose $p$ according to Theorem~\ref{thm_infinite}. 
Remember that $p$ provides a trade-off between the first term and the second term of the sub-optimality bound, i.e., $pH$ and $\delta\epsilon_bH^2$, and the order of the error depends on the larger term. Therefore, the best order of the bound can be achieved when the two terms are equal. Based on this intuition, the relationship between $p$ and $\delta$ should be $p=\delta \epsilon_b H$. In fact, the choice of $p$ preserves the $\tilde{\gO}(\epsilon_bH)$ bound in Corollary~\ref{corollary} when the assumption on $P$ is satisfied. Please refer to Appendix~\ref{sec_p} for a detailed discussion. 

It is natural to assume the optimization process is smooth, i.e., the intervention probability $\delta$ and policy 0-1 loss $\epsilon_b$ changes slowly throughout the optimization process. Therefore, we can calculate $p$ using $\delta$ and $\epsilon_b$ of the last iteration as an approximation. $\delta$ and $\epsilon_b$ of the last iteration are easy to obtain because $\epsilon_b$ can be calculated directly and $\delta$ can be estimated with the intervention frequency.

For tasks with continuous action
spaces, the policy 0-1 loss is exactly 1, which makes the bound in Theorem~\ref{thm_infinite} trivial. In fact, Theorem~\ref{thm_infinite} holds for $\ell$ is the TV divergence between $\pi$ and $\pi^*$, and we discuss this in Appendix~\ref{sec_p}. According to Pinsker's inequality~\citep{pinsker}, $\TV(P,Q)\leq \sqrt{\KL(P,Q)}$, i.e., KL divergence can be the upper bound of TV divergence. Thus we use KL divergence instead to avoid the complex computation of TV divergence because the condition $\mathbb E_{s\sim \beta}[\ell(s, \pi, \pi')]\leq \epsilon_b$ still holds when $\ell$ is selected as the TV divergence of policy and $\epsilon_b$ is selected as the KL divergence of policy.
In this way, we only need to determine $p$ in the first iteration. The key idea to tune the initial $p$ is to let $p$ approximately equals $\delta\epsilon_b H$, which can be easily calculated after a few interactions with environments.

\subsection{Surrogate of Q-value difference}
In many real-world applications, though the exact expert Q-values are hard to get upfront, many existing methods can acquire a Q that is close to $Q^{\pi^*}$, including learning from offline datasets~\cite{cql, combo, hve}, using human advice~\cite{PEBBLE, human+}, and computing from rules~\cite{kernel, knowledge}. 
However, in some cases the Q-function cannot be obtained, we hope to find a surrogate of $Q^{\pi^*}$. Note that the expert policy $\pi^*$ is accessible, and we derive the relationship between Q-value difference and policy divergence as follows. 
\begin{theorem}\label{thm_pi}
The Q-value difference can be bounded by the policy divergence:
$$Q^{\pi^*}_h(s,\pi^*)-Q^{\pi^*}_h(s,\pi)\leq D_{\textnormal{TV}}(\pi^*(\cdot|s), \pi(\cdot|s))(H-h).$$
\end{theorem}
This theorem shows $D_{\textnormal{TV}}(\pi^*(\cdot|s), \pi(\cdot|s))(H-h)$ is the upper bound of $Q^{\pi^*}_h(s,\pi^*)-Q^{\pi^*}_h(s,\pi)$. 
Using the upper bound as a surrogate is reasonable because the sub-optimality bound in Theorem~\ref{thm_infinite} is preserved.

Similarly, in the environments with continuous action spaces, we use $\sqrt{\KL(\pi^*(a|s),\pi(a|s))}$ instead of $\TV(\pi^*(a|s),\pi(a|s))$. This is because TV divergence is difficult to calculate in continuous action spaces, and Pinsker's inequality~\citep{pinsker} guarantees the theoretical results under this modification.

For our practical algorithm, as the threshold is adaptively tuned in the training process, we name it \textbf{A}ctive a\textbf{da}ptive ex\textbf{p}ert involve\textbf{Men}t (\methodname).
The pseudo-code of \methodname\ is given in Alg.~\ref{alg}.

\begin{algorithm}[htbp]
    \caption{Training procedure of \methodname}
    \label{alg}
    \begin{algorithmic}[1]
    \REQUIRE {
    An expert policy $\pi^*$;
    A Q-function $Q$ corresponding to $\pi^*$;
      Number of sampling steps $N$; 
      Learner update interval $K$
    }
    \STATE Initialize  learner policy $\pi$, buffer $B$, $p$
    \FOR{$n = 1$ to $N$}{
        \STATE Get learner agent action $a_l$ and expert action $a_e$
        \STATE Calculate the surrogate Q-value difference $D_Q$
        \IF {$D_Q > p$}{
            \STATE Take expert action $a_e$, add the transition to $B$
            }
        \ELSE{
           \STATE  Take learner agent action $a_l$
        } \ENDIF
        \IF {$n\%K == 0$}{
            \STATE Sample batches of transitions from $B$ to train $\pi$ 
            \STATE Update $p$-value
        }\ENDIF
        }
    \ENDFOR
    \end{algorithmic}
\end{algorithm}

\section{Experiments}\label{sec_experiment}

\begin{figure*}[htbp]
  \centering
  \subfigure[Performance]{
    \label{fig:perf_test_return}
    \includegraphics[width=0.64\linewidth, trim={5 0 0 0}, clip]{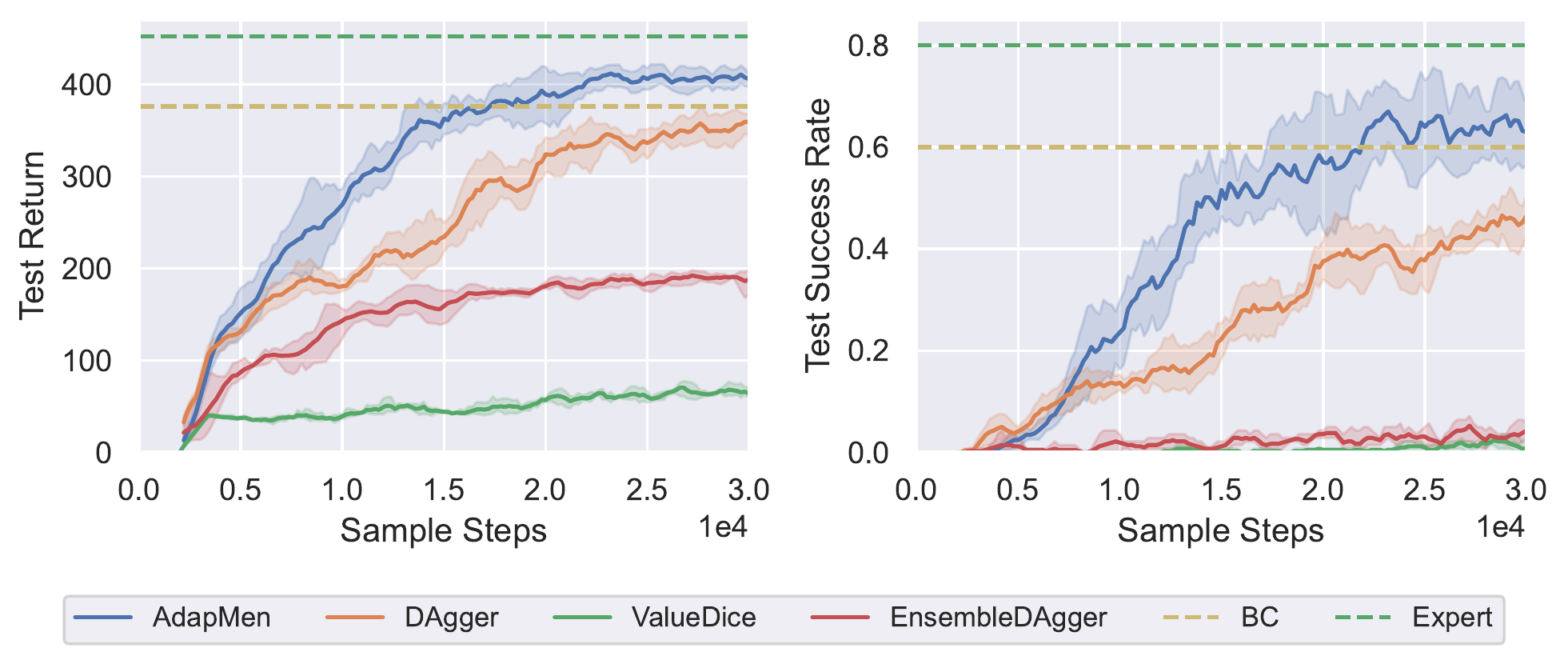}
  }
  \subfigure[Total expert action usage]{
  \vspace{3mm}\includegraphics[width=5.8cm, height=4.9cm, trim={0 0 0 30}, clip]{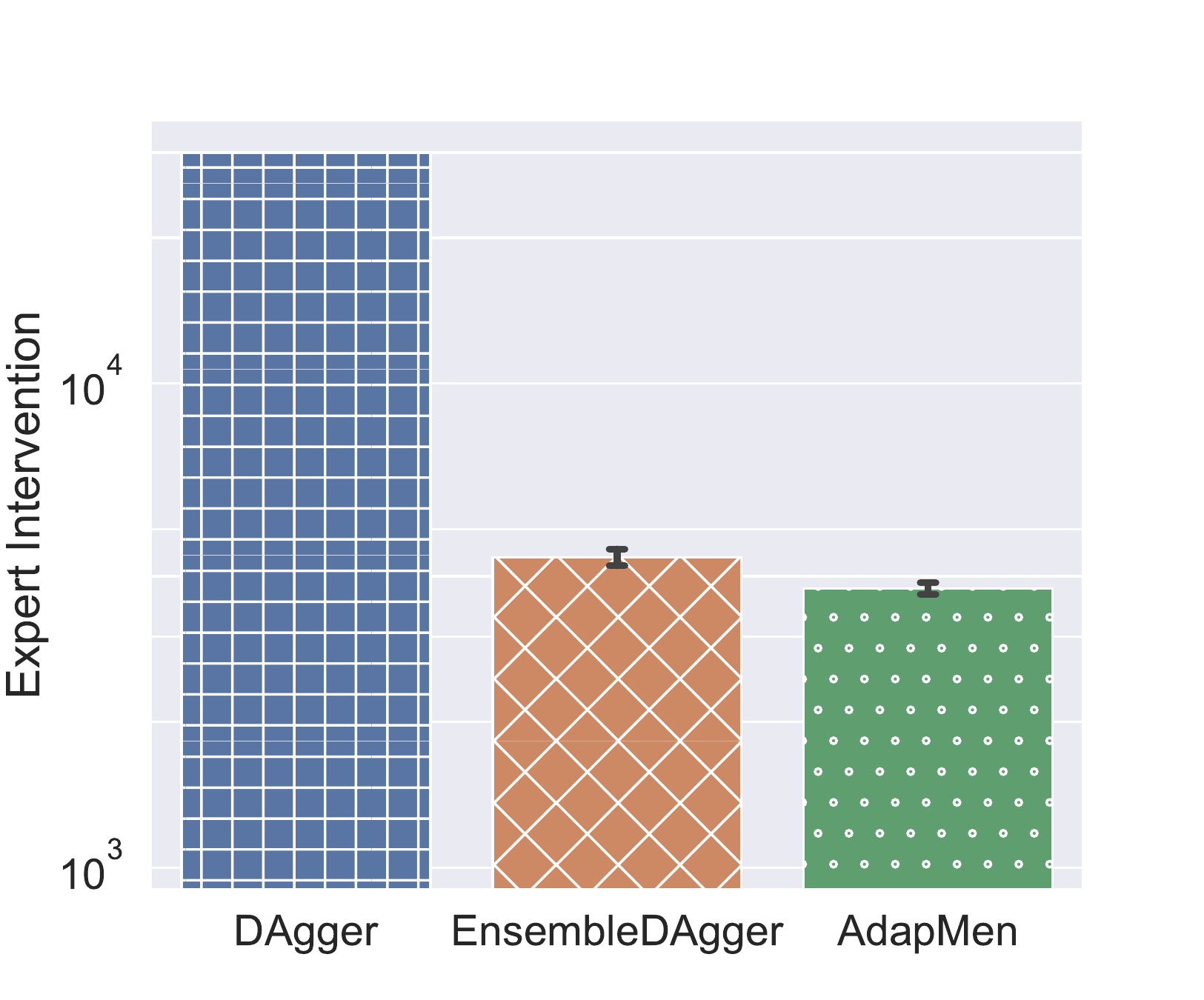}
    \label{fig:intervention_count}
  }
  \caption{Performance in MetaDrive with policy experts}
  \label{fig:performance}
\end{figure*}

In this section, we conduct experiments to test whether \methodname\ reaches the theoretical advantages of our framework. 

We choose MetaDrive~\citep{metadrive} and Atari 2600 games from ALE ~\citep{ALE} as benchmarks. MetaDrive is a highly compositional autonomous driving benchmark that is closely related to real-world applications. The MetaDrive simulator can generate an infinite number of diverse driving scenarios from both procedural generation and real data importing. The agent observes a 259-dimensional vector which is composed of a 240-dimensional vector denoting the 2D-Lidar-like point clouds, a vector summarizing the target vehicle’s state and a vector for the navigation information. The action space is a continuous 2-dimensional vector representing the acceleration and steering of the car, respectively. The goal is to follow the traffic rules and reach the target position as fast as possible. The training configuration of MetaDrive follows that in ~\citep{EGPO}. For the justice of comparison, the evaluation is performed on 20 randomly selected scenarios. 
Atari 2600 games are challenging visual-input RL tasks with discrete action spaces. Using conventional environment wrappers and processing techniques, the agent observes a $(84\times 84)$ grayscale image and has discrete action spaces ranging from 6 valid actions to 18 valid actions depending on the game. We randomly select six common Atari games. To avoid the small stochasticity problem of the Atari simulator, we activate the "sticky action" feature to simulate actual human input and increase stochasticity.

We choose BC~\citep{BC1}, DAgger~\citep{dagger}, HG-DAgger~\citep{HG-dagger}, EnsembleDAgger~\citep{ensembledagger}, and ValueDICE~\citep{ValueDICE} as baselines.
The details of BC and DAgger have been introduced in Sec.~\ref{sec_intro}. HG-DAgger and EnsembleDAgger are representative methods of active imitation learning methods.
HG-DAgger allows interactive imitation learning from human experts in real-world systems by letting a human expert take over control when deemed necessary, and
EnsembleDAgger uses both action variance from an ensemble of policies and action discrepancies between learner and expert as the criterion to decide whether the expert should take over control.
ValueDICE is the SoTA of AIL methods, which trains the learner agent via robust divergence minimization in an off-policy manner.
Hyper-parameters of the implementations of baselines are listed in Appendix~\ref{sec:parameter}.

We first test the performance of \methodname\ and baselines in the two benchmarks with expert in the form of trained policies, namely policy experts. Then, we dive into \methodname\ and demonstrate some key properties of our algorithm to answer the following questions:
\begin{itemize}
\item How is the intervention threshold automatically adjusted during the training process?
\item Can the distribution of $D_Q$ satisfy the assumption in Corollary~\ref{corollary} in most cases?
\item Is policy divergence a good surrogate of $D_Q$?
\end{itemize}
Finally, we simulate real-world scenarios by letting a human be the expert and control the vehicle in the MetaDrive benchmark. 

\begin{figure*}[t]
  \centering
    \includegraphics[width=1.0\linewidth]{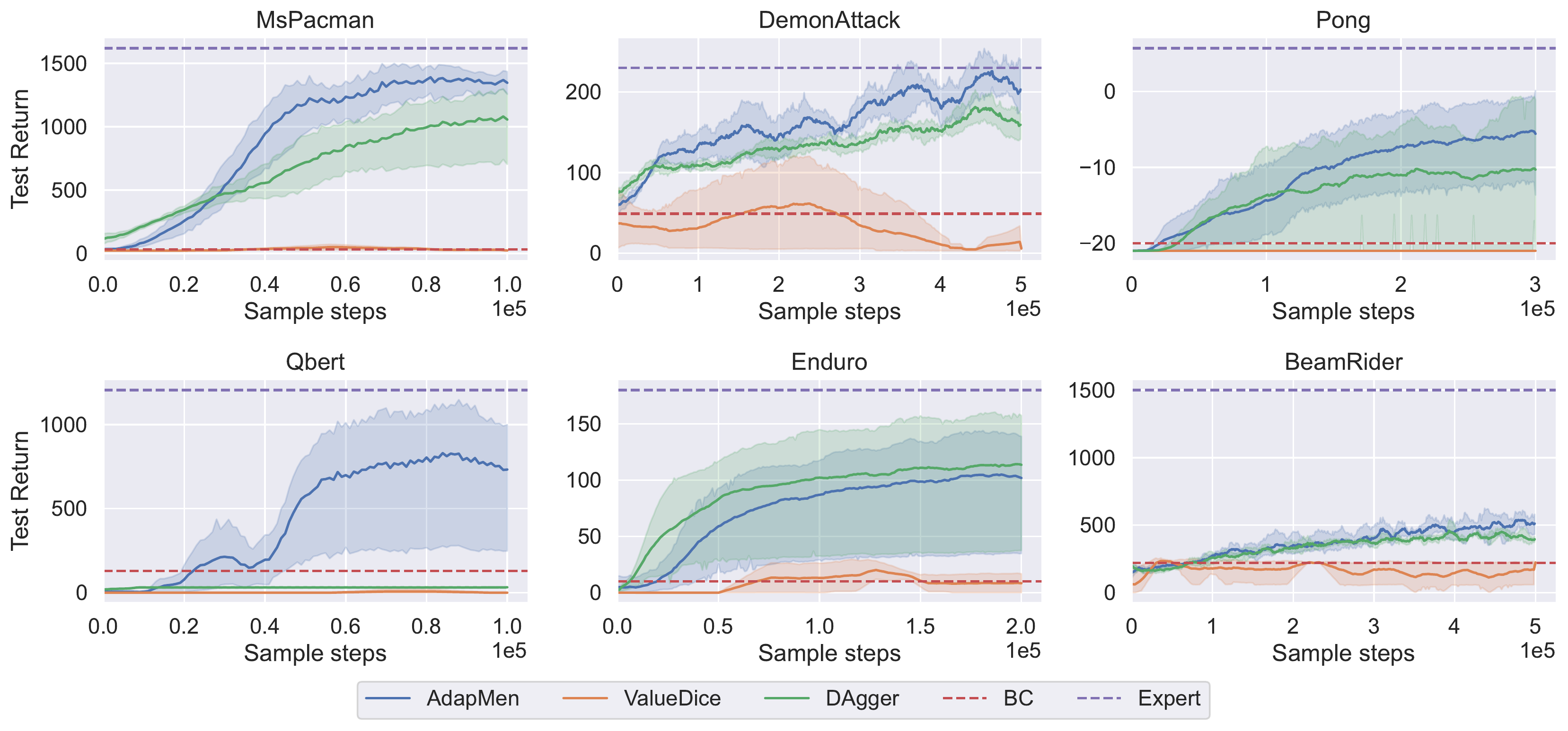}
  \caption{Performance in six Atari games with policy experts}
  \label{fig:performance_atari}
\end{figure*}

\subsection{Performance with policy experts}
\subsubsection{Performance in MetaDrive}

The expert policy of MetaDrive is trained by Soft Actor-Critic (SAC) ~\citep{SAC}. For \methodname, we take one of the trained Q-networks as $Q^{\pi^*}$ and calculate $D_Q$ based on it. To demonstrate the robustness towards inaccurate $Q^{\pi^*}$ when the ground truth value is not available, we also perform experiments on the estimated value function in Appendix~\ref{sec_extra}.

The performance in the MetaDrive benchmark is plotted in Fig.~\ref{fig:performance}. The horizontal axis represents the total number of steps sampled in the environment. The vertical axis of Fig.~\ref{fig:perf_test_return} and Fig.~\ref{fig:intervention_count} are policy return, success rate, and number of expert interventions, respectively.
HG-DAgger is omitted in experiments for the sake of fairness because the expert of the algorithm should be human.

For the MetaDrive benchmark, \methodname\ achieves the best performance in terms of both cumulative return and success rate. ValueDICE achieves the worst performance probably because of its highest sample complexity and sensitivity to hyper-parameters thus we fail to find a working configuration. 
Notwithstanding the low expert intervention counts of EnsembleDAgger, the performance of EnsembleDAgger severely degrades.
The $\mu$-recoverability property of DAgger is hard to satisfy in risk-sensitive environments so that DAgger shows no advantage than BC. 
BC achieves the best performance among all the baseline algorithms. This is because the policy expert has little stochasticity and the dimension of input is small.

The total number of expert data usage is shown in Fig.~\ref{fig:intervention_count}. Here the expert data usage is defined as the number of expert state-action pairs added to the buffer for training the learner. This quantity of BC, ValueDICE is the same as DAgger and we omit them in the figure. DAgger always adds the expert state-action pair to the buffer, thus having the biggest expert data usage.
Compared with EnsembleDAgger, by generating the best buffer distribution for teaching, \methodname\ requires fewer expert interventions while achieving a better test performance.

To further verify our theory, we draw the trend of $p$-value and actual intervention probability throughout the training process in Fig.~\ref{fig:pvalue}, where the left vertical axis represents the value of $p$,while the right vertical axis represents the intervention probability. The probability is calculated  every 200 sample steps in the environment. Theorem~\ref{thm_infinite} implies the sub-optimality is negatively related to $p$ and $\delta$. This is verified by the decreasing trend of $p$ and $\delta$ in the training process, coinciding with the increasing policy return in Fig.~\ref{fig:perf_test_return}. Meanwhile, the sharply changing $p$ also demonstrates the importance of adaptively changing intervention criterion. Intuitively, as the learner agent gets better at driving the car, the teacher should increase the difficulty of the teaching policy. A lower $p$-value indicates more difficult learning content.
\begin{figure}[htbp]
  \centering
    \includegraphics[width=0.6\linewidth,  trim={10 0 0 30}, clip]{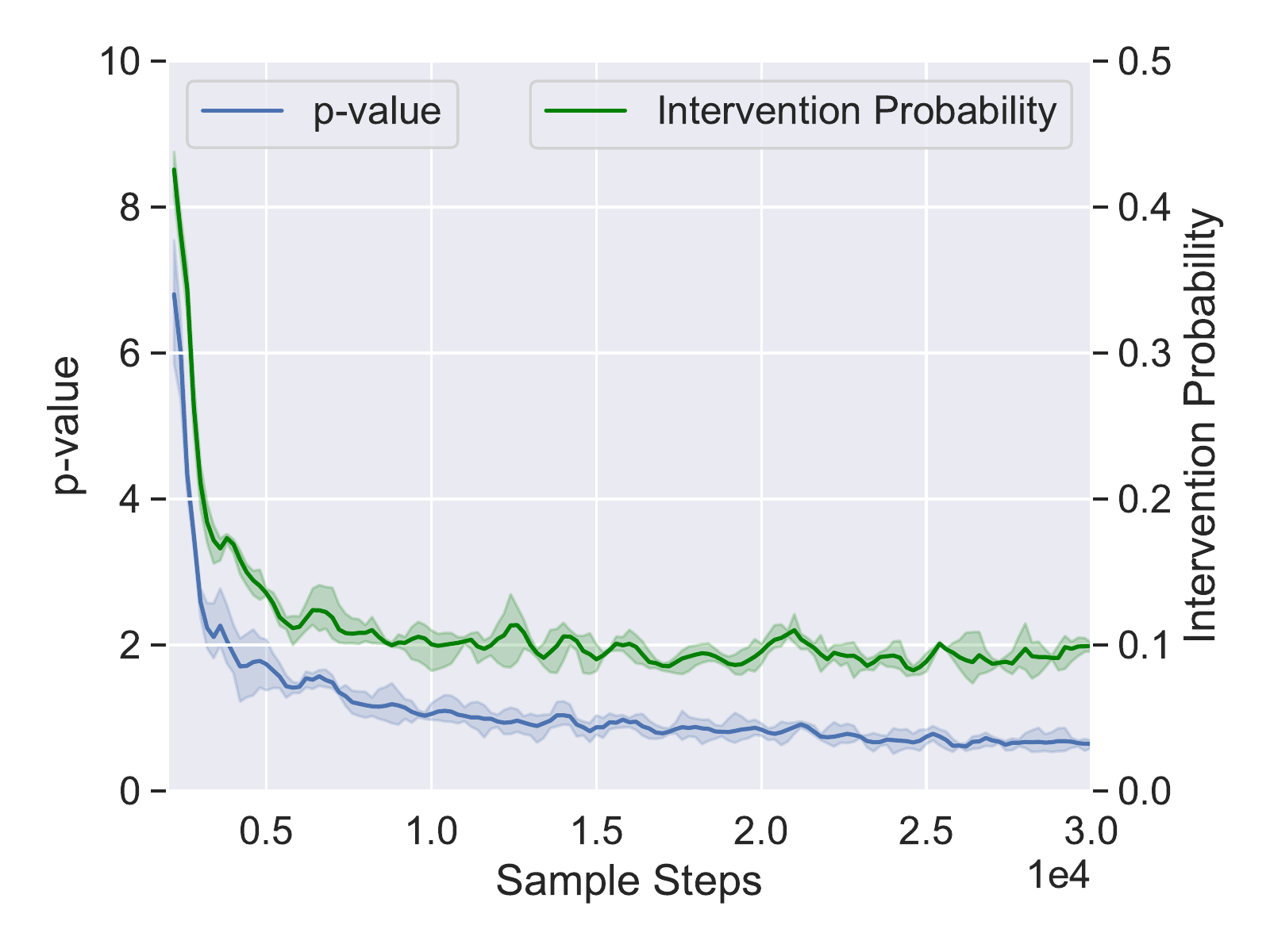}
    \caption{$p$-value and intervention probability of \methodname\ on MetaDrive}     \label{fig:pvalue}
\end{figure}

\vspace{-1.5mm}
\subsubsection{Performance in Atari games}

The expert policies of Atari games are trained by Deep Q Learning~\citep{atari_dqn}. Note that we activate the "sticky action" features to increase the stochasticity of the tasks.
Since Ensemble-DAgger requires a continuous action space, we omit it for comparison in the Atari 2600 games. The performance curves in the Atari 2600 games are plotted in Fig.~\ref{fig:performance_atari}.

\methodname\ outperforms baselines in 5 out of 6 Atari games. These tasks are more challenging, which can be inferred from the performance of baselines. In Qbert, all algorithms fail to learn from the expert except for \methodname. In all the tasks, ValueDICE performs equally poorly as in MetaDrive. BC, which has near-optimal performance in MetaDrive,  also collapses in most of the six Atari games. This shows that BC fails in higher-dimensional  environments. DAgger performs 
better than other baselines, this is probably because the $\mu$-recoverability assumption can still be satisfied in most states in Atari games.

\subsection{Performance of \methodname\ criterion based on policy divergence}
\begin{figure}[htp]
  \centering
        \includegraphics[width=1.0\linewidth]{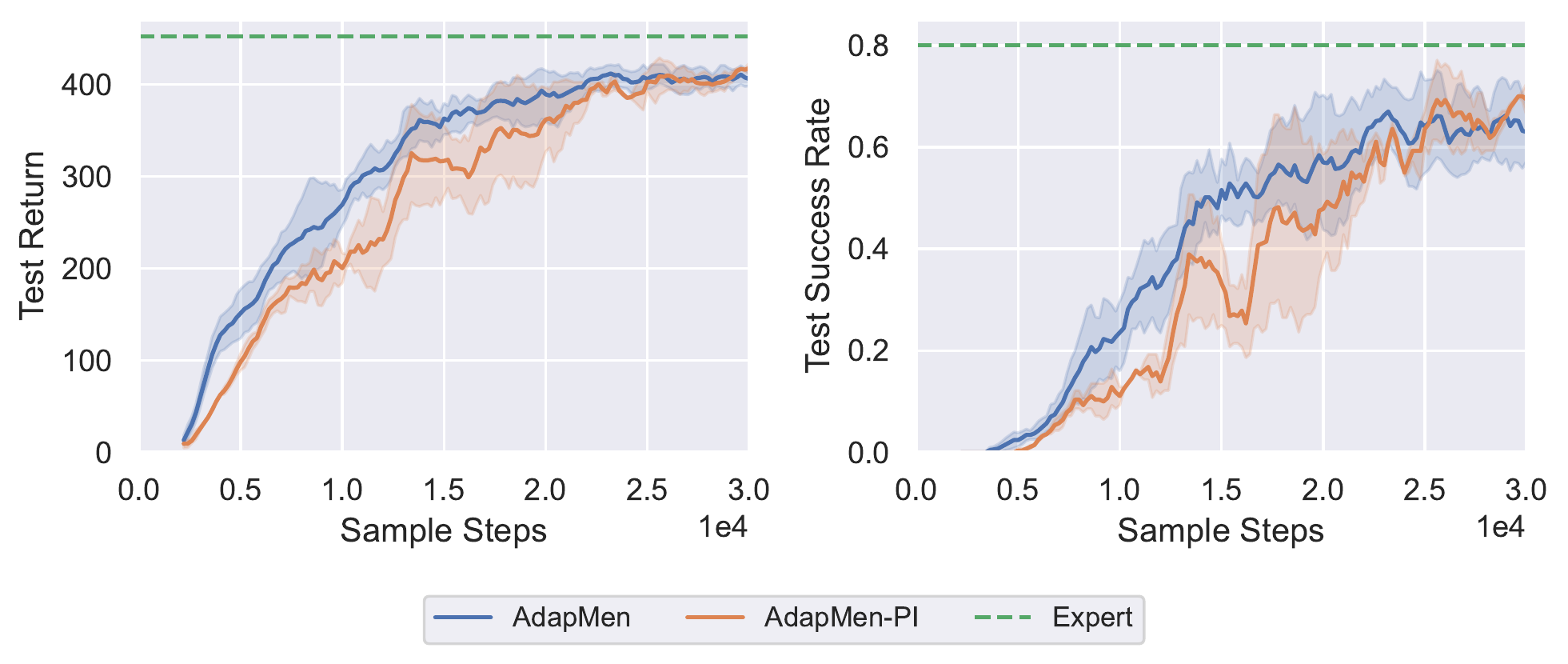}
      \caption{Performance in MetaDrive with different criteria of \methodname}
      \label{fig:perf_adapmen}
\end{figure}

As mentioned in Sec.~\ref{sec_practical}, when $Q^{\pi^*}$ is not available, we use policy divergence 
as a surrogate of $D_Q$. To validate the correctness of this surrogate, we test it on MetaDrive, and plot its performance in Fig.~\ref{fig:perf_adapmen}. \methodname\ is the original algorithm, while \methodname-PI uses the policy divergence instead of $D_Q$. The result shows \methodname-PI has comparable performance with \methodname. This experiment validates our theory and demonstrates that policy divergence is also a proper criterion.

\subsection{Analysis of $D_Q$ distribution}

\begin{figure}[htbp]
  \centering
  \subfigure[MetaDrive-initial]{
    \includegraphics[width=0.475\linewidth, trim={10 20 5 10}, clip]{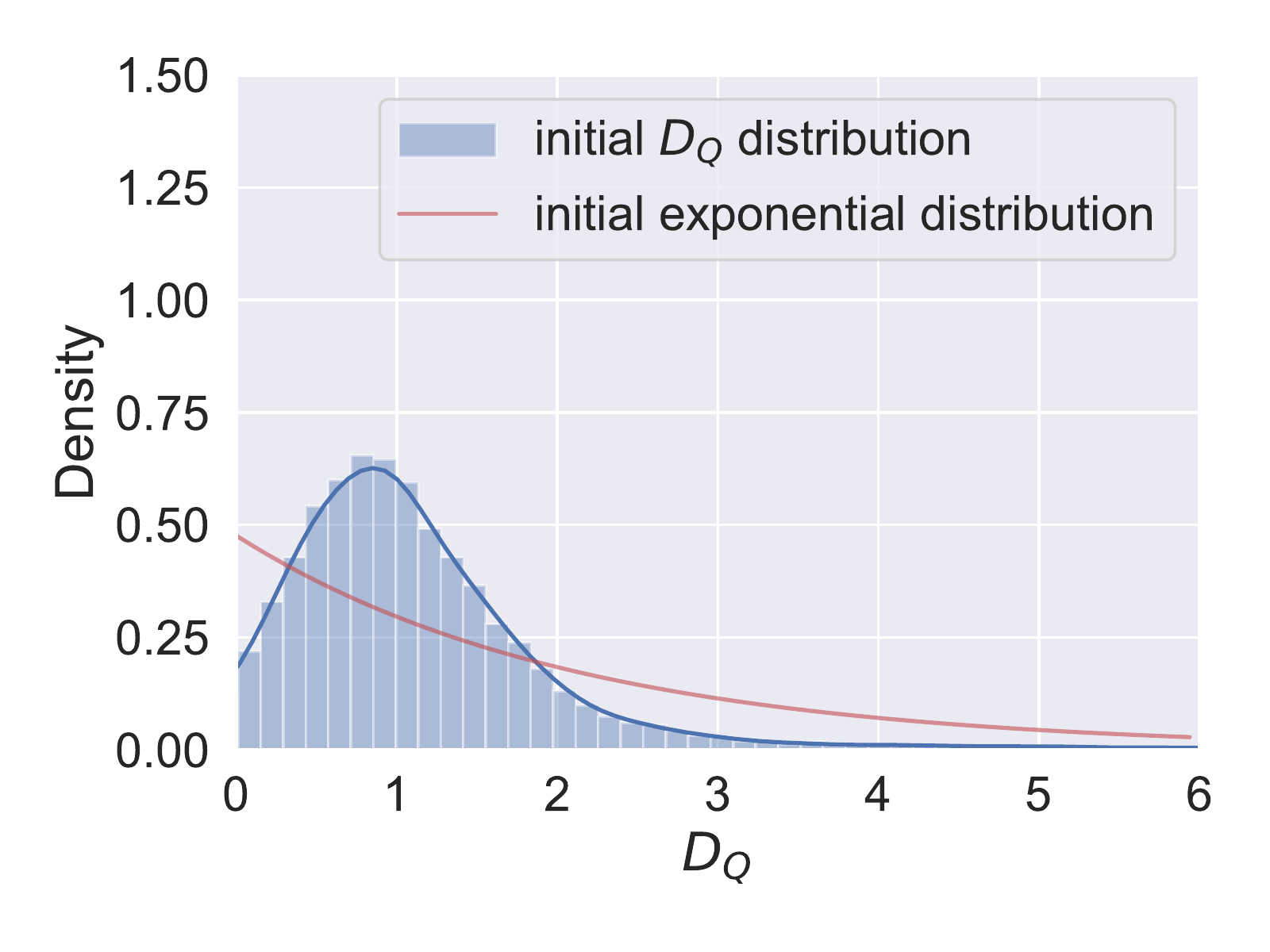}
  }
  \subfigure[MetaDrive-final]{
    \includegraphics[width=0.475\linewidth]{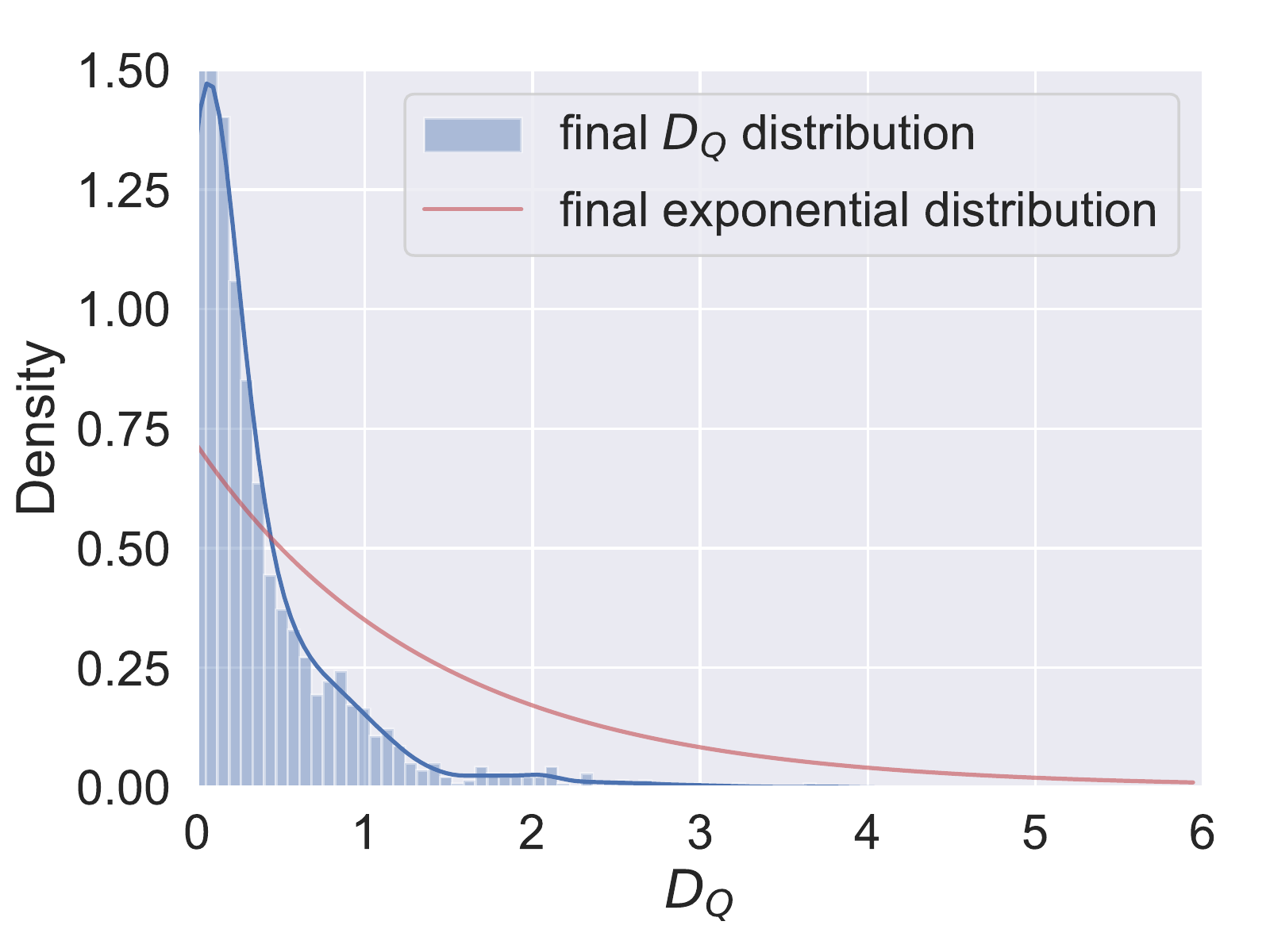}
  }
  \vspace{-1.5mm}
  \subfigure[DemonAttack-initial]{
    \includegraphics[width=0.475\linewidth, trim={10 20 5 10}, clip]{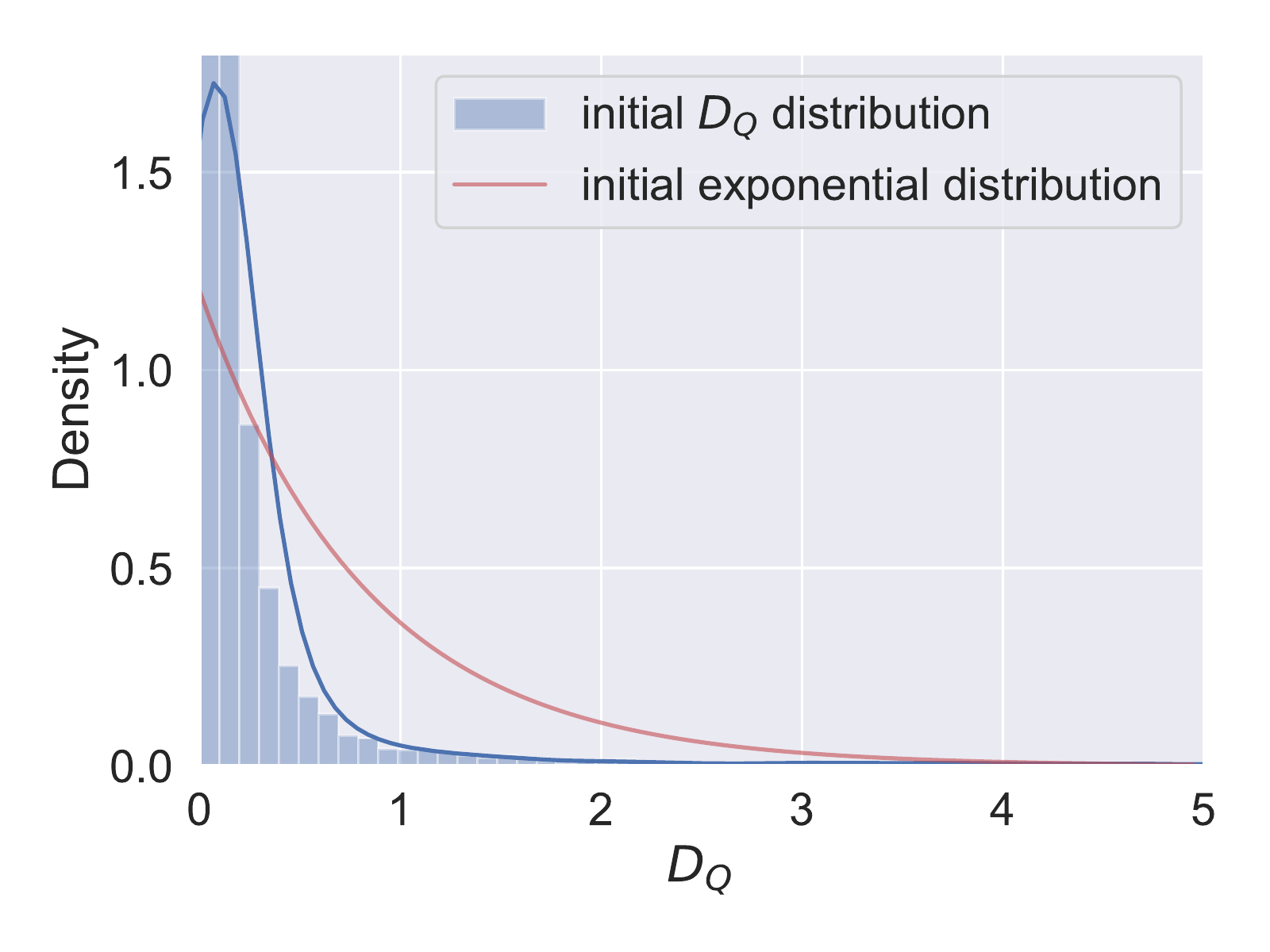}
  } 
  \subfigure[DemonAttack-final]{
    \includegraphics[width=0.475\linewidth, trim={10 20 5 10}, clip]{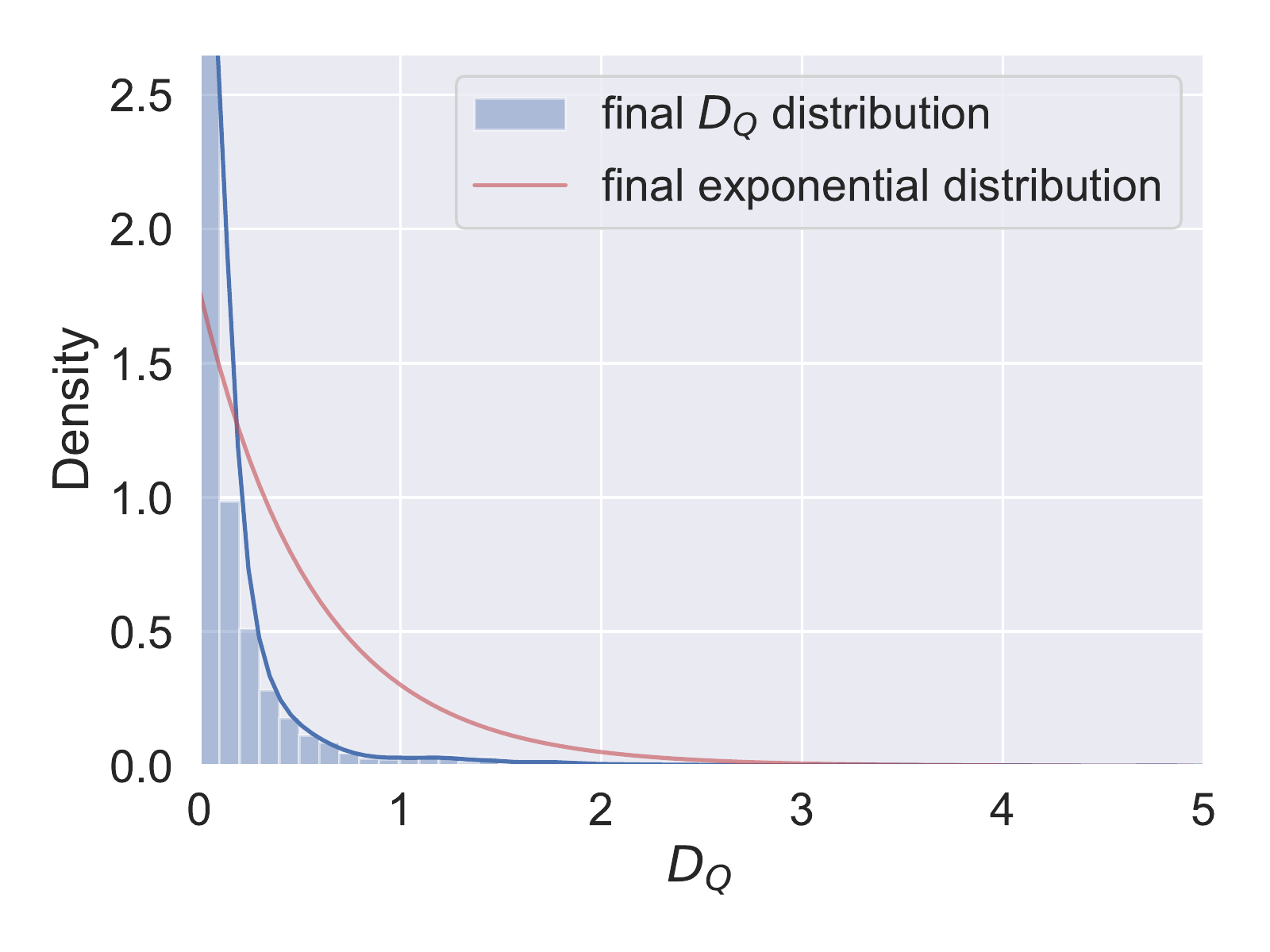}
  } 
  \vspace{-1.5mm}
  \subfigure[MsPacman-initial]{
    \includegraphics[width=0.475\linewidth, trim={10 20 5 10}, clip]{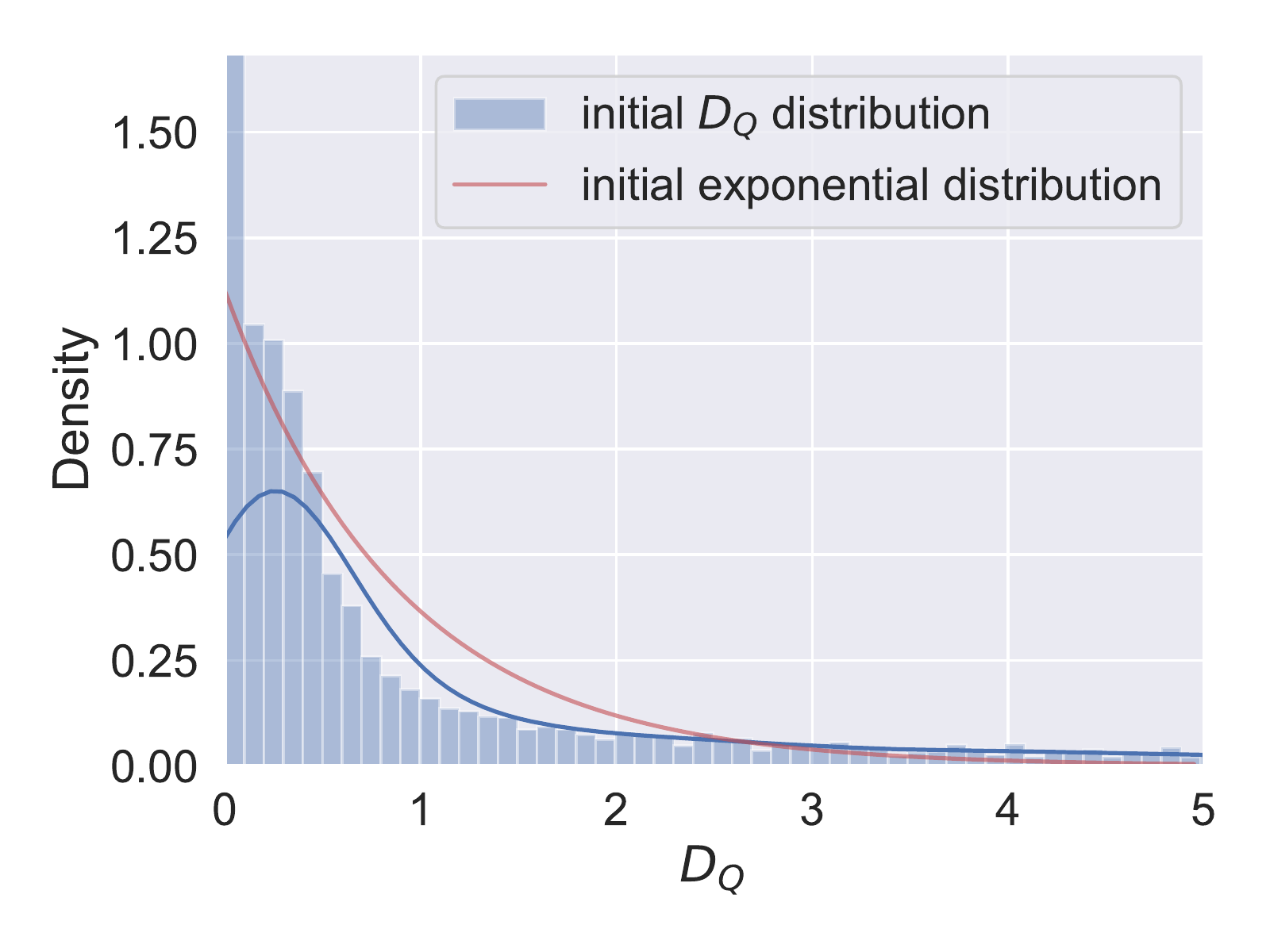}
  } 
  \subfigure[MsPacman-final]{
    \includegraphics[width=0.475\linewidth, trim={10 20 5 10}, clip]{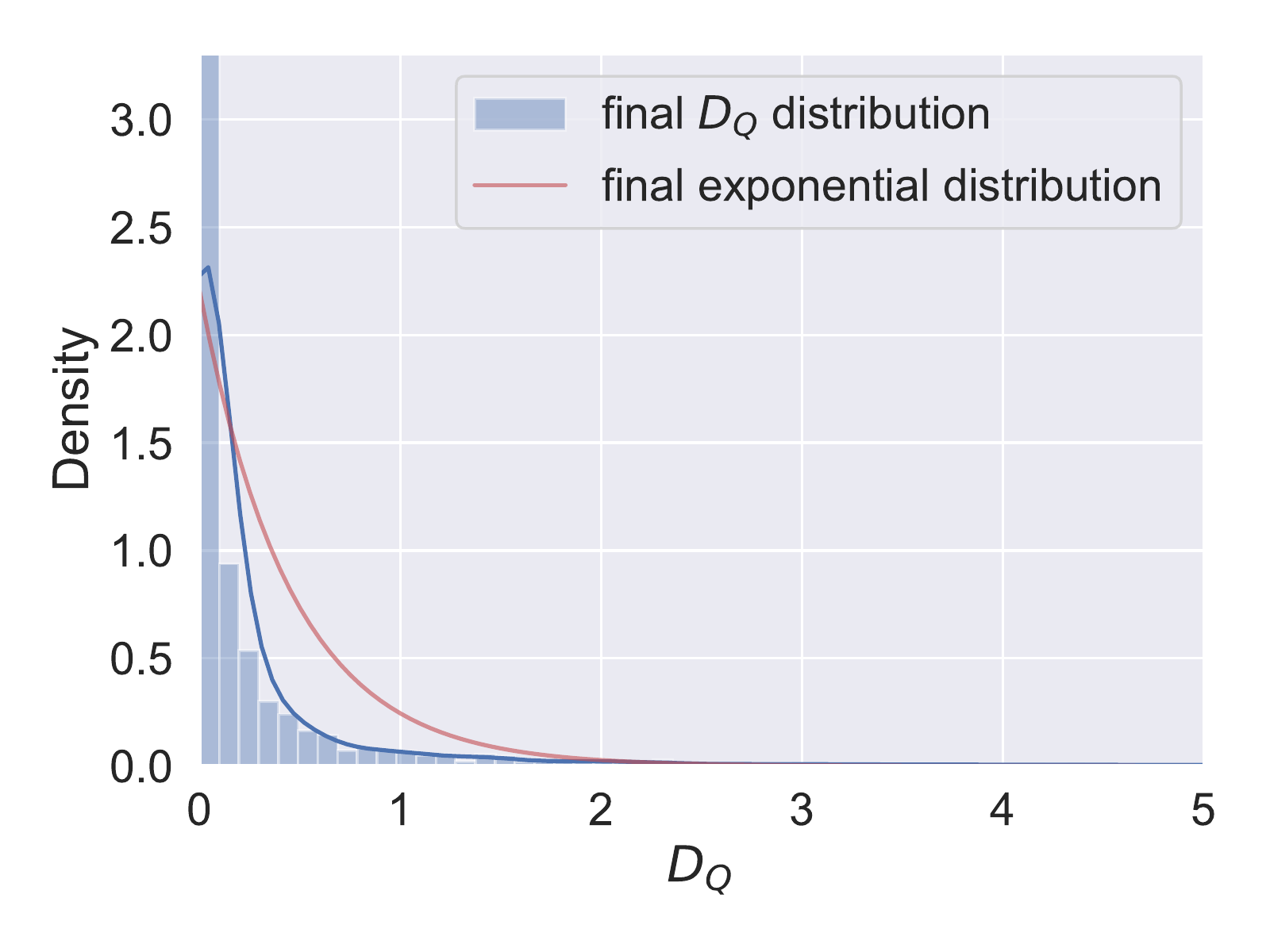}
  } 
  \vspace{-1.5mm}
  \subfigure[BeamRider-initial]{
    \includegraphics[width=0.475\linewidth, trim={10 20 5 10}, clip]{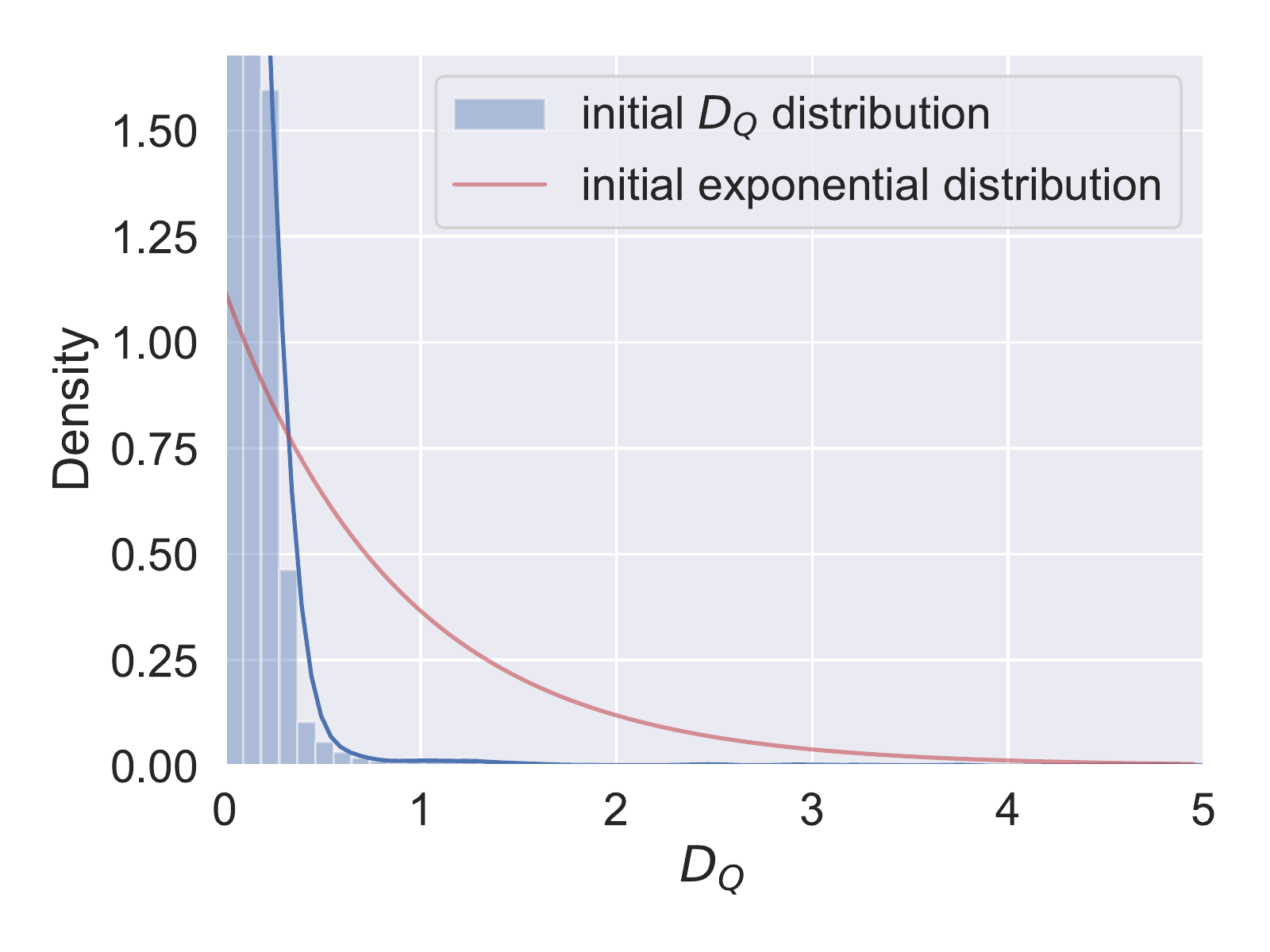}
  } 
  \subfigure[BeamRider-final]{
    \includegraphics[width=0.475\linewidth, trim={10 20 5 10}, clip]{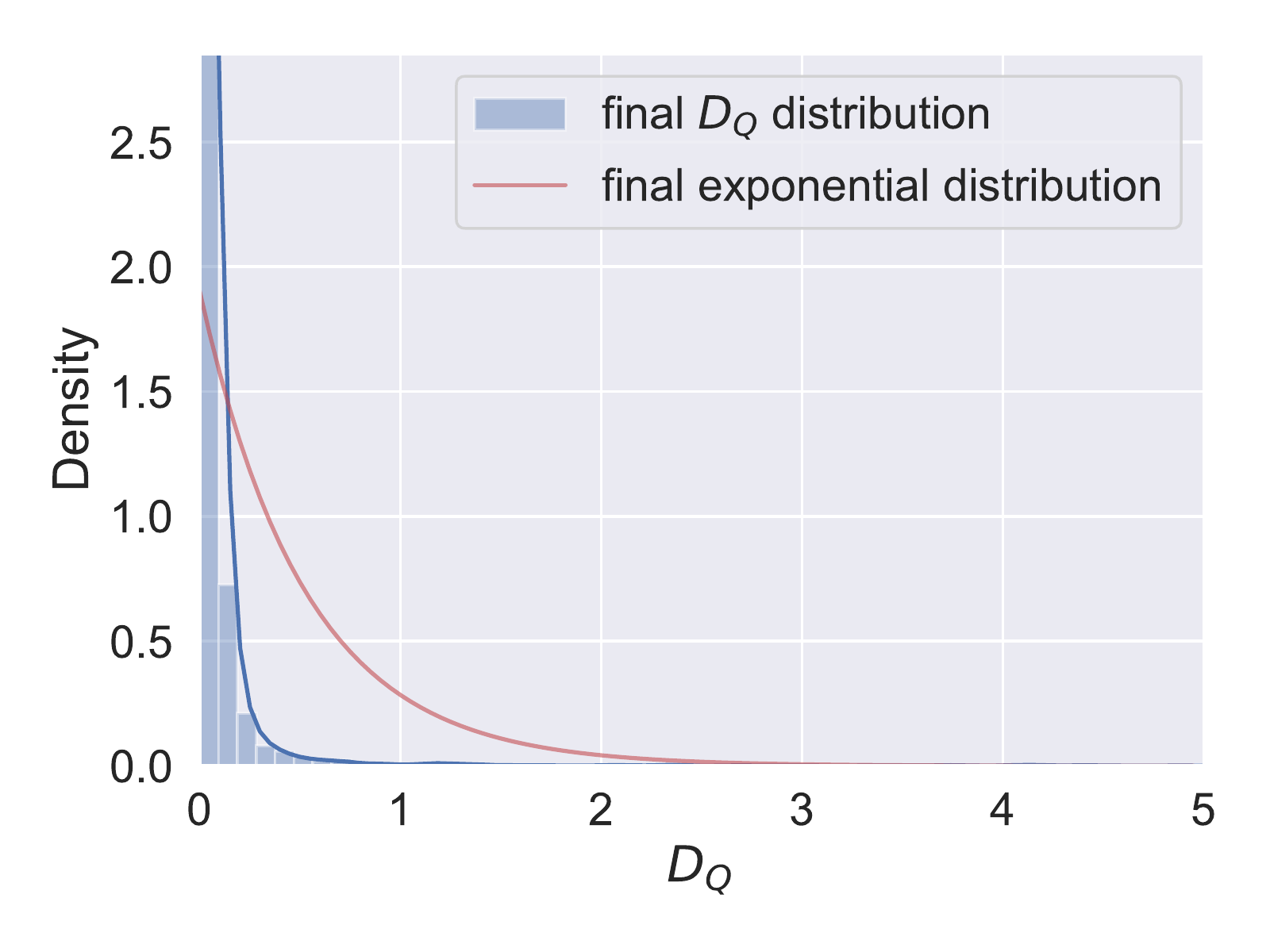}
  } 
  \vspace{-1.5mm}
  \caption{$D_Q$ distributions in MetaDrive and Atari games. The blue lines show the distributions of $D_Q$ estimated by kernel density estimation.}
  \label{fig:analysis}
  \end{figure}
  
In Corollary~\ref{corollary}, we assume the distributions of $D_Q$, i.e., $P$ in Sec.\ref{sec_analysis}, belongs to $\gO(\epsilon_b)$-Sub-Exponential distribution class with expectation $\gO(\epsilon_b)$. The assumption is satisfied if the tail of $P$ is bounded by an exponential distribution with parameter $\epsilon_b$. To verify this assumption, we plot $P$ of MetaDrive and the six Atari games and compare their tails with exponential distribution. The partial results are shown in Fig.~\ref{fig:analysis} because of space limitation. The rest of the results are in Appendix~\ref{sec_dq}. The distributions at the beginning and at the end of the training are on the left-hand side and right-hand side, respectively. All the tails of $P$ are bounded by the exponential distribution, which implies the assumption is satisfied in nearly all tested tasks. This bridges the gap between the theoretical analysis and practical applicability of \methodname.

\subsection{Performance in MetaDrive with a human expert}
  \begin{figure}[htbp]
  \centering
      \subfigure[Test Return]{
        \label{fig:perf_test_return_human}
        \includegraphics[width=0.5\linewidth]{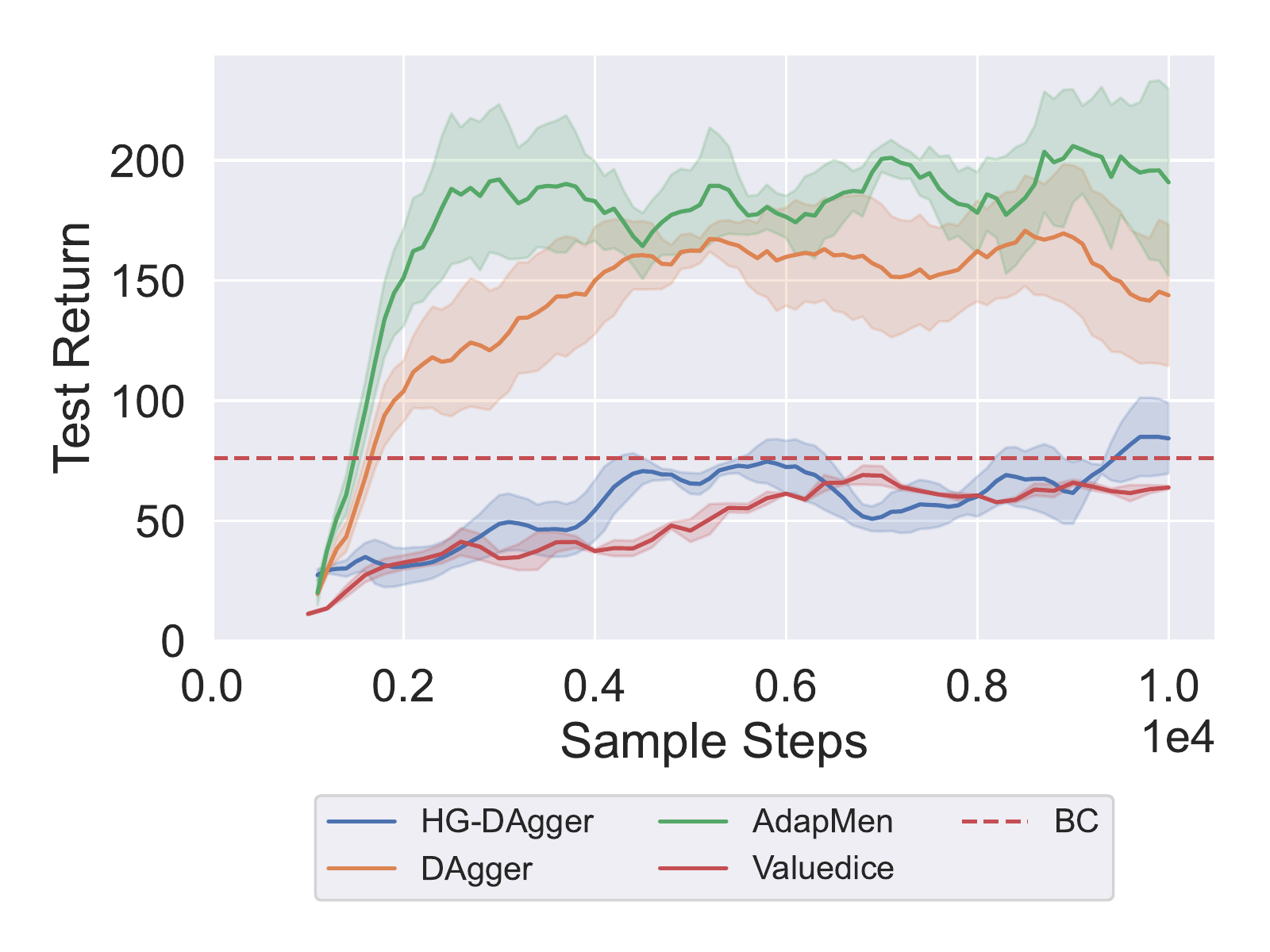}
      }
      \subfigure[Expert action usage]{
        \label{fig:perf_success_human}
        \includegraphics[width=0.45\linewidth]{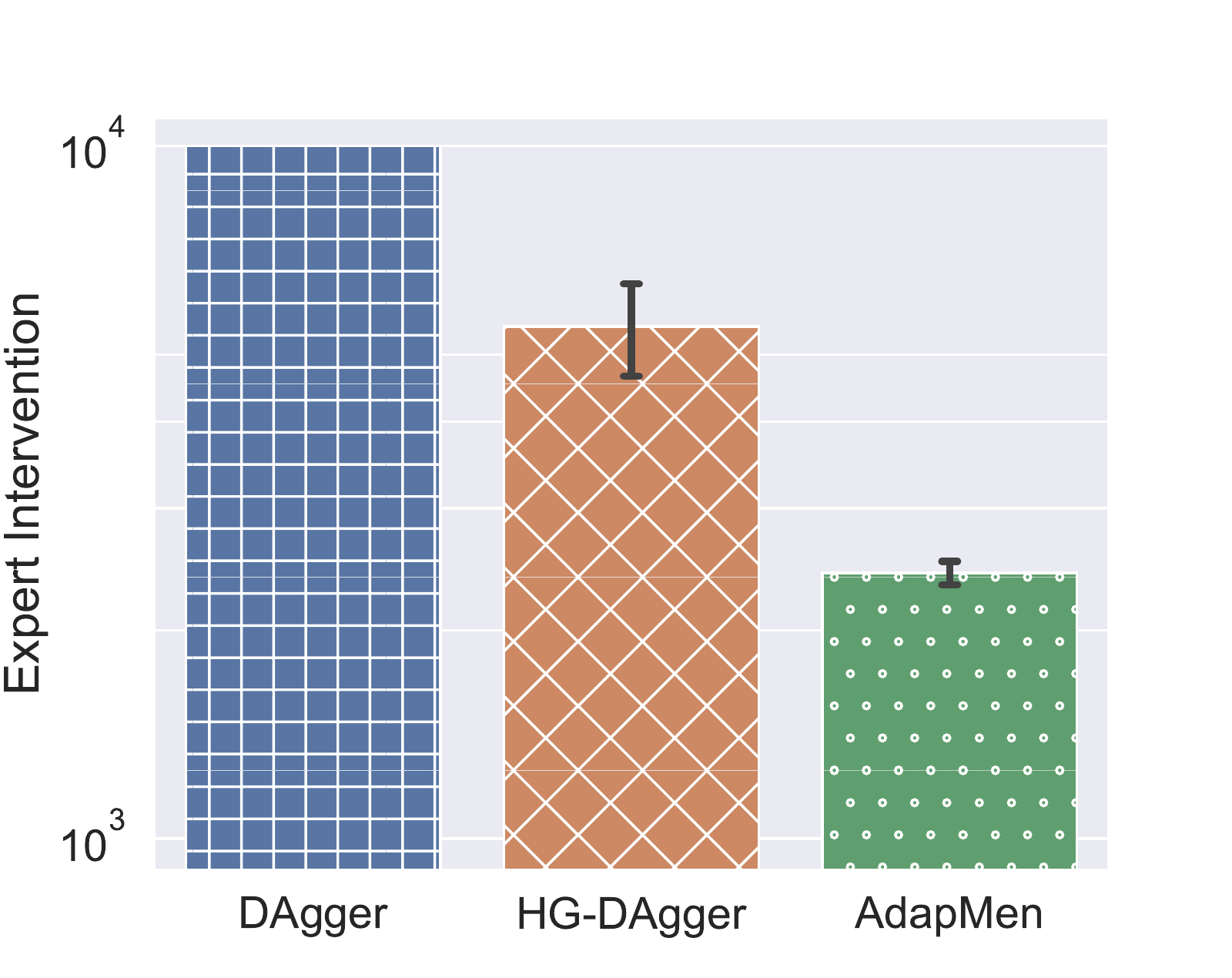}
      }
      \caption{Performance in MetaDrive with a human expert}
      \label{fig:perf_human}
\end{figure}
In real-world tasks, humans are important sources of expert information, especially in autonomous driving tasks. To mimic real-world tasks, we substitute the SAC expert in MetaDrive with a human. The experimental results are shown in Fig.~\ref{fig:perf_human}. The random and sometimes irrational behaviors of  human experts raise a huge challenge for imitation learning algorithms, and general degradation of performance happens for all methods. BC has a 75\% performance degradation. In contrast, \methodname\ has a relatively small performance degradation and achieves the best final performance. The performance of HG-DAgger is surprising. Although our human expert has tries his best to correct the behavior of the learner agent,  HG-DAgger is only slightly better than BC. HG-DAgger even uses more expert actions to train the learner policy than \methodname. This shows the teaching strategy of humans are unreliable and an objective criterion is important. 

\section{Conclusion}
In this paper, we formulate the IL process as a teacher-student interaction framework. The proposed framework shows expert should involve in the interaction of the agent with the environment according to a certain criterion. We theoretically verify the effectiveness of this framework, and derive a better error bound and sample complexity  under a mild condition, which we experimentally demonstrate common in many benchmarks. 
Based on the teacher-student interaction framework, we propose a practical method \methodname, where the intervention criterion is tuned automatically in the training process, which frees the hyper-parameter tuning budget of other active imitation learning methods. Experimental results demonstrate that \methodname\ achieves a better performance than other IL methods.



\section*{Acknowledgement}
This work is supported by National Key Research and Development Program of China (2020AAA0107200), the National Science Foundation of China (61921006, 62276126) and the Major Key Project of PCL (PCL2021A12).

\bibliographystyle{ACM-Reference-Format} 
\bibliography{sample}


\appendix
\onecolumn
\section{Proofs of Section~\ref{sec_analysis} and Section~\ref{sec_practical}}\label{sec_proof}

For brevity, we denote $\pi(\cdot|s)$ as $\pi(s)$ in the appendix.
\subsection{Proof of Theorem~\ref{thm_infinite}}
\begin{lemma}[Safety]\label{lem_safty}
The teaching policy $\pi'$ satisfies $J(\pi')\geq J(\pi^*)-pH$.
\end{lemma}
\begin{proof}
We follow a similar proof to \citep{dagger}. Given our policy $\pi'$, consider the policy $\pi'_{0:t}$, which executes $\pi'$ in the first t-steps and then executes the expert policy $\pi^*$. Then
\begin{align*}
    J(\pi')&=J(\pi^*)+\sum_{t=0}^{H-1}[J(\pi'_{1:H-t})-J(\pi'_{1:H-t-1})]\\
    &=J(\pi^*)+\sum_{t=1}^H\mathbb E_{s\sim d^{\pi'}_t}[Q_{t}^{\pi^*}(s,\pi')-Q_{t}^{\pi^*}(s,\pi^*)]\\
    &\geq J(\pi^*)-pH
\end{align*}
\end{proof}
Besides facilitating the proof of the theorem, this lemma also demonstrates that the deployed policy only suffers from a $pH$ value loss, which ensures safety in the learning process.

\begin{proof}[Proof of Theorem~\ref{thm_infinite}]
Note that $J(\pi^*)-J(\pi)=J(\pi^*)-J(\pi')+J(\pi')-J(\pi)$, where $J(\pi^*)-J(\pi')$ can be derived from Lemma~\ref{lem_safty}. Thus, we only need to calculate $J(\pi')-J(\pi)$. Here we use $\pi_{0:t}$ to denote the policy that executes $\pi$ in the first t-steps and then executes $\pi'$. Then,
\begin{align*}
    J(\pi')&=J(\pi)+\sum_{t=0}^{H-1}[J(\pi_{1:H-t})-J(\pi_{1:H-t-1})]\\
    &=J(\pi)+\sum_{t=1}^H\mathbb E_{s\sim d^{\pi'}_t}[Q_{t}^{\pi}(s,\pi')-Q_{t}^{\pi}(s,\pi)]\\
    &\overset{(a)}{=}J(\pi)+\sum_{t=1}^H\mathbb E_{s\sim d^{\pi'}_t}\mathbb I(Q_{t}^{\pi^*}(s,\pi')-Q_{t}^{\pi^*}(s,\pi)>p)[Q_{t}^{\pi}(s,\pi')-Q_{t}^{\pi}(s,\pi)]\\
    &=J(\pi)+\sum_{t=1}^H\mathbb E_{s\sim d^{\pi'}_t}\mathbb I(Q_{t}^{\pi^*}(s,\pi')-Q_{t}^{\pi^*}(s,\pi)>p)\mathbb I(\pi(s)\neq \pi^*(s))[Q_{t}^{\pi}(s,\pi')-Q_{t}^{\pi}(s,\pi)]\\
    &\overset{(b)}{\leq}J(\pi)+H\sum_{t=1}^H\mathbb E_{s\sim d^{\pi'}_t}\mathbb I(Q_{t}^{\pi^*}(s,\pi')-Q_{t}^{\pi^*}(s,\pi)>p)\mathbb I(\pi(s)\neq \pi^*(s))\\
    &\overset{(c)}{=}J(\pi)+\delta\epsilon_b H^2,
\end{align*}
where (a) uses the fact that $\pi'=\pi$ if $Q_{t}^{\pi^*}(s,\pi')-Q_{t}^{\pi^*}(s,\pi)<p$, (b) is because the Q-value is no more than $T$, and (c) is derived from $\mathbb E_{s\sim \beta}[\ell(s, \pi, \pi^*)]=\mathbb E_{s\sim \beta}[\ell(s, \pi, \pi')]\leq \epsilon_b$.
\end{proof}

\subsection{Proof of Corollary~\ref{corollary}}
\begin{proof}[Proof of Corollary~\ref{corollary}]
Under the sub-exponential assumption on $P$, we have
$$\textnormal{Pr}\left(Q^{\pi^*}_t(s,\pi')-Q^{\pi^*}_t(s,\pi)>p\right)\leq \exp(-\frac{p-\mu}{\sigma}),$$
where $\mu=\E[Q^{\pi^*}_t(s,\pi')-Q^{\pi^*}_t(s,\pi)]$ and $\sigma^2=\sV[Q^{\pi^*}_t(s,\pi')-Q^{\pi^*}_t(s,\pi)]$.

Under the condition of the corollary, $\mu=\gO(\epsilon_b)$ and $\sigma=\gO(\epsilon_b)$. Note that 
$$\delta=\E_{s\sim d_{\pi'}}\sI(Q^{\pi^*}_t(s,\pi')-Q^{\pi^*}_t(s,\pi)>p)=\textnormal{Pr}\left(Q^{\pi^*}_t(s,\pi')-Q^{\pi^*}_t(s,\pi)>p\right),$$
we have $\delta=\gO(\frac{1}{H})$ if $p\geq \mu+\sigma\log H=\Omega(\epsilon_bH)$.

Thus,
$$J(\pi^*)-J(\pi)\leq pH+\delta \epsilon_b H^2 =\gO(\epsilon_bH\log H)+\gO(\epsilon_bH)= \tilde{\gO}(\epsilon_b H).$$
\end{proof}

\subsection{Proof of Theorem~\ref{thm_finite}}

Let $\{\hatpi'_i\}_{i=1}^N$ be the teaching policy induced by $\{\hatpi_i\}_{i=1}^N$ and the intervention threshold $p$.

\begin{lemma}\label{lemma_avarage}
Suppose there is an online learning algorithm which outputs policies $\{\hatpi_1, \dots, \hatpi_N\}$ sequentially according to any procedure. Here the learner learns the policy $\hatpi_i$ from some distribution conditioned on $\Tr_1, \dots, \Tr_{i-1}$ and subsequently samples a trajectory $\Tr_i$ by rolling out $\hatpi'_i$, the policy depending on $\hatpi_i$. This process is repeated for $N$ iterations. Denote $\hat{\beta}_i$ as the empirical distribution over the states induced by $\Tr_i$. Denote $\hatpi=\frac{1}{N}\sum_{i=1}^N\hatpi_i$ as the mixture policy. Let $\E_{s\sim \hat{\beta}_i}[\ell(s,\hatpi,\pi^*)]=\gL(\hat{\beta}_i, \hatpi, \pi^*)$ and $\hat{\beta}=\frac{1}{N}\hat{\beta}_i$, the mixture state distribution. Then,
$$\E[\gL(\beta, \hatpi, \pi^*)]=\frac{1}{N}\sum_{i=1}^N\E[\gL(\hat{\beta}_i, \hatpi_i,\pi^*)].$$
\end{lemma}
\begin{proof}
Since the trajectory $\Tr_i$ is rolled out using $\hatpi'_i$. Conditioned on $\hatpi_i$, $\hat{\beta}_i$ is unbiased and equal to $\beta_i$ in expectation because $\hatpi'_i$ is derived from $\hatpi$. Therefore, for each $i$, since $\gL(\beta, \hatpi, \pi^*)$ is a linear function of $\beta$, 
$$\E\left[\gL(\hat{\beta}_i,\hatpi_i,\pi^*)|\hatpi_i\right]=\gL(\beta_i,\hatpi_i,\pi^*).$$
Summing across $i=1,\cdots,N$ and using the fact that $\gL(\beta,\hatpi,\pi^*)=\frac{1}{N}\sum_{i=1}^N\gL(\beta_i,\hatpi_i,\pi^*)$, taking the expectation completes the proof.
\end{proof}

The lemma implies it suffices to minimize the empirical 0-1 loss under the empirical distribution. 

Note that
$$\gL(\hat{\beta}_i,\pi,\pi^*)=\frac{1}{H}\sum_{t=1}^H\sum_{s\in\gS}\left<\pi^t(\cdot|s),z_i^t(s)\right>,$$
where $z_i^t=\left\{\hat{\beta}_i(s)(1-\pi^*(\cdot|s))\right\}$. 

To learn the returned policy sequence, we use the normalized-EG algorithm~\citep{book}, which is also known as online mirror descent with entropy regularization for online learning. Formally, the online learning problem and the algorithm are described in Section 2 of~\citep{book}.

\begin{lemma}[Theorem 8 in~\cite{value_interaction}]\label{lemma_omd}
Assume that the normalized EG algorithm is runned in a sequence of linear loss functions $\{\left<z_i,\cdot\right>:i=1, \cdots, T\}$, with $\eta=1/2$ to return a sequence of distributions $w_1, \cdots, w_T\in \Delta_{\gA}^1$. Assume that for all $t\in[H]$, $\bm 0\preceq z_t\preceq\bm 1$. For any $u$ such that $\sum_{i=1}^T\left<z_i,u\right>=0$, 
$$\sum_{t=1}^T\left<w_i-u,z_i\right>\leq 4\log(|\gA|).$$
\end{lemma}

\begin{proof}[Proof of Theorem~\ref{thm_finite}]
According to Lemma~\ref{lemma_omd}, we have
$$\sum_{t=1}^H\left<z_i^t(s),\hatpi_i(\cdot|s)\right>\leq 4\log(|\gA|).$$
Averaging across $t\in [H]$, summing across $s\in\gS$, and rescalling the definition of $\gL$ result in
$$\frac{1}{N}\sum_{i=1}^N\gL(\hat{\beta}_i,\hatpi_i,\pi^*)\leq \frac{4||\gS|\log|\gA|)}{N}.$$
From Lemma~\ref{lemma_avarage}, the resulting sequence of policies $\hatpi_1,\dots,\hatpi_N$ and their mixtures $\hatpi$ satisfies,
$$\E\left[\gL(\beta,\hatpi,\pi^*)\right]\leq \frac{4|\gS|\log(|\gA|)}{N}.$$

This implies $\epsilon_b\leq \frac{4|\gS|\log(|\gA|)}{N}$. Note that $$\E_{s\sim d^{\hatpi'}}\sI(Q_{t}^{\pi^*}(s,\hatpi')-Q_{t}^{\pi^*}(s,\hatpi)>p)=\frac{1}{N}\sum_i\delta_i,$$
Combining this results with Theorem~\ref{thm_infinite}, we complete the proof.

\end{proof}

\subsection{Proof of Theorem~\ref{thm_pi}}

\begin{lemma}\label{lem_q_diff}
The Q-value difference satisfies
$$Q^{\pi^*}_h(s_h,\pi^*)-Q^{\pi^*}_h(s_h,\pi)\leq D_{\textnormal{TV}}(d_1(s,a),d_2(s,a))(H-h+1),$$
where $d_1(s,a)=\frac{1}{H-h+1}\sum_{t=h}^H\textnormal{Pr}(s_t=s,a_t=a|s_h, a_t\sim \pi^*)$ and $d_2(s,a)=\frac{1}{H-h+1}\sum_{t=h}^H\textnormal{Pr}(s_t=s,a_t=a|s_h, a_h\sim \pi, a_t\sim \pi^*, \forall t> h)$.
\end{lemma}
\begin{proof}
\begin{align*}
    Q^{\pi^*}_h(s_h,\pi^*)-Q^{\pi^*}_h(s_h,\pi)&=(H-h+1)\sum_{s,a}(d_1(s,a)-d_2(s,a))r(s,a)\\
    &\leq (H-h+1)\sum_{s,a}(d_1(s,a)-d_2(s,a))_+r(s,a)\\
    &\leq (H-h+1)\TV(d_1(s,a),d_2(s,a)),
\end{align*}
where $(p_1(s,a)-p_2(s,a))_+=p_1(s,a)-p_2(s,a)$ if $p_1(s,a)-p_2(s,a)>0$, otherwise $(p_1(s,a)-p_2(s,a))_+=0$. The last inequality results from the assumption that $r(s,a)\leq 1$.
\end{proof}
\begin{lemma}[Lemma B.1 of \citep{mbpo}]\label{lem_mbpo}
Suppose $p_1(s,a)=p_1(s)p_1(a|s)$ and $p_2(s,a)=p_2(s)p_2(a|s)$, we can bound the total variation distance of the joint as:
$$\TV(p_1(s,a),p_2(s,a))\leq \TV(p_1(s),p_2(s))+\E_{s\sim p_1}\TV(p_1(a|s),p_2(a|s)).$$
\end{lemma}

\begin{proof}[Proof of Theorem~\ref{thm_pi}]
According to Lemma~\ref{lem_q_diff}, we only need to bound $\TV(d_1(s,a),d_2(s,a))$.

Let $d_1(s)=\frac{1}{H-h+1}\sum_{t=h}^H\textnormal{Pr}(s_t=s|s_h, a_t\sim \pi^*)$ and $d_2(s)=\frac{1}{H-h+1}\sum_{t=h}^H\textnormal{Pr}(s_t=s|s_h, a_h\sim \pi, a_t\sim \pi^*)$. Let $d_i^t(s)$ be the state distribution at step $t$, i.e., $d_i(s)=\frac{1}{H-h+1}\sum_{t=h}^Hd_i^t(s)$. Similarly, let $d_i^t(s,a)$ be the state action distribution at step $t$. Then $d_1^t(s,a)=d_1^t(s)\pi^*(a|s)$ and $d_2^t(s,a)=d_2^t(s)\pi^*(a|s)$ for $t\geq h+1$, and $d_1^t(s,a)=d_1^t(s)\pi^*(a|s)$, $d_2^t(s,a)=d_2^t(s)\pi(a|s)$ for $t=h$. We apply Lemma~\ref{lem_mbpo} to $d_i^t(s,a)$ and obtain:
\begin{align}
&\TV(d_1^t(s,a),d_2^t(s,a))\leq \TV(d_1^t(s),d_2^t(s))+\E_{s\sim d_1^t(s)}\TV(\pi^*(a|s),\pi^*(a|s))=\TV(d_1^t(s),d_2^t(s)), \qquad t\geq h+1  \label{eq_after_step}\\ 
&\TV(d_1^t(s,a),d_2^t(s,a))\leq \TV(d_1^t(s),d_2^t(s))+\E_{s\sim d_1^t(s)}\TV(\pi^*(a|s),\pi(a|s))=\TV(\pi^*(a|s_h),\pi(a|s_h)), \qquad t=h \label{eq_first_step}
\end{align}
The derivation of the second inequality uses the fact that the state at step $h$ is exactly $s_h$.

Therefore, we only need to focus on the TV divergence between $d_1^t(s)$ and $d_2^t(s)$. For $t>h+1$,
\begin{align*}
    \TV(d^t_1(s),d^t_2(s))&=\frac{1}{2}\sum_s\left|d_1^t(s)-d_2^t(s)\right|\\
    &=\frac{1}{2}\sum_s \left|\sum_{s',a'}\left(d_1^{t-1}(s')\pi^*(a'|s')\gP(s|s',a')-d_2^{t-1}(s')\pi^*(a'|s')\gP(s|s',a')\right)\right|\\
\end{align*}
Denote $\sum_{a'}\pi^*(a'|s')\gP(s|s',a')$ as $p(s|s')$, the above equation can be simplified as
\begin{align*}
    \TV(d^t_1(s),d^t_2(s))&=\frac{1}{2}\sum_s \left|\sum_{s'}\left(d_1^{t-1}(s')p(s|s')-d_2^{t-1}(s')p(s|s')\right)\right|\\
    &\leq\frac{1}{2}\sum_{s,s'}\left|d_1^{t-1}(s')p(s|s')-d_2^{t-1}(s')p(s|s')\right|\\
    &=\frac{1}{2}\sum_{s,s'}p(s|s')\left|d_1^{t-1}(s')-d_2^{t-1}(s')\right|\\
    &\overset{(a)}{=}\frac{1}{2}\sum_{s'}\left|d_1^{t-1}(s')-d_2^{t-1}(s')\right|\\
    &=\TV(d_1^{t-1}(s'),d_2^{t-1}(s')),
\end{align*}
where (a) uses $\sum_sp(s|s')=1$. Recursively, we have $\TV(d^t_1(s),d^t_2(s))=\TV(d^{h+1}_1(s),d^{h+1}_2(s))$ for all $t\geq h+1$. 

For $\TV(d^{h+1}_1(s),d^{h+1}_2(s))$, we have
\begin{align*}
    \TV(d^{h+1}_1(s),d^{h+1}_2(s))&=\frac{1}{2}\sum_s|d_1^{h+1}(s)-d_2^{h+1}(s)|\\
    &=\frac{1}{2}\sum_s \left|\sum_{s',a'}\left(d_1^{h}(s')\pi^*(a'|s')\gP(s|s',a')-d_2^{h}(s')\pi(a'|s')\gP(s|s',a')\right)\right|\\
    &=\frac{1}{2}\sum_s \left|\sum_{a'}\left(\pi^*(a'|s_h)\gP(s|s_h,a')-\pi(a'|s_h)\gP(s|s_h,a')\right)\right|\\
    &\leq \frac{1}{2}\sum_{s,a'}\left|\pi^*(a'|s_h)\gP(s|s_h,a')-\pi(a'|s_h)\gP(s|s_h,a')\right|\\
    &=\frac{1}{2}\sum_{a'}\left|pi^*(a'|s_h)-\pi(a'|s_h)\right|\\
    &=\TV(\pi^*(a|s_h),\pi(a|s_h)).
\end{align*}
Then
\begin{align*}
    \TV(d_1(s,a),d_2(s,a))&=\TV\left(\frac{1}{H-h+1}\sum_{t=h}^Hd_1^t(s,a),\frac{1}{H-h+1}\sum_{t=h}^Hd_2^t(s,a)\right)\\
    &\leq \frac{1}{H-h+1}\sum_{t=h}^H\TV(d_1^t(s,a),d_2^t(s,a))\\
    &\overset{(a)}{=}\frac{1}{H-h+1}\left(\TV(\pi^*(a|s_h),\pi(a|s_h))+\sum_{t=h+1}^H\TV(d_1^t(s),d_2^t(s))\right)\\
    &=\TV(\pi^*(a|s),\pi(a|s)),
\end{align*}
where (a) uses Eq.~(\ref{eq_first_step}) and (\ref{eq_after_step}). Finally, using Lemma~\ref{lem_q_diff} completes the proof.
\end{proof}

\section{Review of Previous Imitation Learning Methods}\label{sec_review}
\textbf{Behavioral Cloning.} BC ignores the changes between the train and test distributions and simply trains a policy $\pi$ that performs well under the distribution of states $d_{\pi^*}$ encountered by the expert policy. This is achieved by the standard supervised learning:
$$\hat{\pi}=\argmin_\pi \E_{s\sim d_{\pi^*}}[\ell(s,\pi, \pi^*)].$$

Assuming $\ell(s,\pi, \pi^*)$ is the 0-1 loss, i.e., $\ell(s,\pi, \pi^*)=\textnormal{Pr}(\pi(s)\neq \pi^*(s))$, we have the following sub-optimality bound in the infinite sample case:
\begin{theorem}[Theorem 2.1 in \citep{bc}]\label{thm_bc}
Let $\mathbb E_{s\sim d_{\pi^*}}[\ell(s,\pi, \pi^*)]=\epsilon_b$, then $J(\pi^*)-J(\pi)\leq H^2\epsilon_b$.
\end{theorem}
The dependency on $H^2$ implies the issue of compounding error. Note that this bound is tight, as \cite{xu2020} pointed out.  Due to the quadratic growth in $H$, BC has a poor performance guarantee.

In the finite sample case, \cite{limits} views BC as the algorithm that outputs a policy belonging to $\Pi_{\textnormal{mimic}}(\gD)$ given expert dataset $\gD$, where $\Pi_{\textnormal{mimic}}(\gD)=\left\{\pi\mid \forall s\in \gD,\pi(s)=\pi^*(s)\right\}$. Let $N$ as be the number of samples in $\gD$ and $\lesssim$ omits the $\log$ term, the sub-optimality bound is:
\begin{theorem}[Theorem 4.2 (a) in \citep{value_interaction}]
Consider any policy $\pi$ which carries out behavioral cloning with expert dataset $\gD$ (i.e., $\pi\in \Pi_{\textnormal{mimic}}(\gD)$), we have $J(\pi^*)-J(\pi)\lesssim \min\left\{H,\frac{|\gS|H^2}{N}\right\}$.
\end{theorem}
The compounding error is also indicated by the factor $H^2$.

\textbf{Distribution Matching.} Rather than considering the IL problem as a policy function approximation, the distribution matching approaches consider the state-action distribution induced by a policy. Concretely, the distribution matching approach proposes to learn $\pi$ by minimizing the divergence between $d^\pi$ and $d^{\pi^*}$. For example, GAIL~\citep{gail} minimizes the JS divergence while ValueDICE~\citep{ValueDICE} minimizes the KL divergence. 

Xu et al.~\cite{xu2020} constructs the sub-optimality bound for distribution matching approaches in the infinite sample case:
\begin{theorem}[Lemma 1 in \citep{ail_finite_sample}]
Let $\pi$ be a policy such that $D_{\textnormal{JS}}(d^{\pi*},d^{\pi})= \epsilon_g$, we have $J(\pi^*)-J(\pi)\leq 2\sqrt{2}\epsilon_g H$.
\end{theorem}
It seems the compounding error issue is solved as the bound is proportional to $H$ rather than $H^2$. However, the two bounds cannot be compared directly as $D_{\textnormal{JS}}(d^{\pi*},d^{\pi})$ is more difficult to optimize than $\mathbb E_{s\sim d^{\pi^*}}[\ell(s,\pi)]$, which implies $\epsilon_g$ can be much larger than $\epsilon_b$ in realistic cases.

To enable a more reasonable comparison, we consider the finite sample case, where the number of samples required to achieve $\epsilon_b$ or $\epsilon_g$ is taken into account. Xu et al.~\cite{ail_finite_sample} give the sub-optimality bound for AIL:
\begin{theorem}[Theorem 1 in \citep{ail_finite_sample}]
Consider the policy $\pi$ generated by AIL with expert dataset $\gD$, we have $J(\pi^*)-J(\pi)\lesssim H\sqrt{\frac{|\gS|-1}{N}}$.
\end{theorem}
Let the sub-optimality gap be $\epsilon$, and the sample complexity of AIL becomes $\tilde{\gO}(|\gS|H^2/\epsilon^2)$, where $\tilde{\gO}$ omits the $\log$ term. However, the sample complexity of BC is $\tilde{\gO}(|\gS|H^2/\epsilon)$, and AIL is even worse than BC in terms of sample complexity. 

\textbf{Dataset Aggregation.} In the DAgger setting, the learner can query the expert when interacting with the environment and trains the next policy under the aggregation of all collected data. Despite the passive result on DAgger that it has the same sub-optimality bound as BC in the worst case given by \citep{limits}, Rajaraman et al.~\cite{value_interaction} prove DAgger can achieve a better performance under the $\mu$-recoverability assumption. 
The definition of the $\mu$-recoverability assumption is:
\begin{definition}[$\mu$-recoverability]
An IL instance is said to satisfy $\mu$-recoverability if for each $t\in [H]$ and $s\in\mathcal{S}$, $Q^{\pi^*}_t(s,\pi^*)-Q^{\pi^*}_t(s,a)\leq \mu$ for all $a\in\mathcal{A}$.
\end{definition}
Satisfying $\mu$-recoverability implies any non-optimal action induces a performance degradation smaller than for every state. It is noteworthy that this quantity is sensitive to the corner case of the environment because the inequality should hold for every state. Furthermore, in many situations where safety is important, a wrong action may cause the failure of a policy because $\mu$ is very large. In the worst case, $\mu=\gO(H)$. Then we state the sub-optimality bound for DAgger.
\begin{theorem}[Theorem 2.2 in \citep{dagger}]\label{thm_dagger}
If an IL instance satisfies $\mu$-recoverability, let $\pi$ be a policy such that $\mathbb E_{s\sim d^\pi}[\ell(s,\pi, \pi^*)]=\epsilon_b$, then $J(\pi)\geq J(\pi^*)-\mu H\epsilon_b$.
\end{theorem}
The result shows DAgger solves the compounding error issue by replacing $H^2$ with $\mu H$. However, $\mu$ can be arbitrarily large and can reduce to the bound of BC in the worst case.

Similarly, we present the result for the finite sample case.
\begin{theorem}[Theorem 1 in \citep{value_interaction}]
Under $\mu$-recoverability, the policy generated in DAgger setting satisfies $J(\pi^*)-J(\pi)\lesssim \mu \frac{|\gS|H}{N}$.
\end{theorem}

\section{Theoretical Guarantee of the Choice of $p$}\label{sec_p}
The choice of $p$ in Sec.\ref{sec_practical} preserves the bound in Corollary~\ref{corollary}.

\begin{theorem}
If the distribution $P$ belongs to $\gO(\epsilon_b)$-Sub-Exponential distribution class with mean $\gO(\epsilon_b)$, and $p=\delta\epsilon_b H$, then $J(\pi^*)-J(\pi)=\gO(\epsilon_bH)$
\end{theorem}
\begin{proof}
According to Corollary~\ref{corollary}, $\delta\leq \frac{1}{H}$ when $p=\mu+\sigma \log H$, where $\mu=\gO(\epsilon_b)$ and $\sigma=\gO(\epsilon_b)$.

Under this circumstance, $\delta\epsilon_b H\leq \epsilon_b< p$, and $p$ should decreases to $p'$ to satisfy the condition $p'=\delta'\epsilon_bH$. Because $p'<p$, $\delta'\epsilon_bH=p'H<pH$. Thus, $p'H+\delta'\epsilon_bH\leq 2pH=\gO(\epsilon_bH)$. Then we complete the proof.

The $\ell$ can be replaced with TV divergence in Theorem~\ref{thm_infinite}.
\end{proof}

\begin{theorem}
Let $\pi$ be a policy such that $\mathbb E_{s\sim \beta}\TV(\pi(s),\pi^*(s))\leq \epsilon_b$, then $J(\pi^*)-J(\pi) \leq pH+\delta \epsilon_b  H^2$, where $\delta=\mathbb E_{s\sim d^{\pi'}}\mathbb I(Q_h^{\pi^*}(s,\pi^*)-Q_h^{\pi^*}(s,\pi)>p)$.
\end{theorem}

\begin{proof}
Lemma~\ref{lem_safty} holds under this circumstance because it does not use the condition on $\ell$.

According to the proof of Theorem~\ref{thm_infinite}, we have
$$J(\pi)= J(\pi^*)-J(\pi')+\sum_{t=1}^H\mathbb E_{s\sim d^{\pi'}_t}\mathbb I(Q_{t}^{\pi^*}(s,\pi')-Q_{t}^{\pi^*}(s,\pi)>p)[Q_{t}^{\pi}(s,\pi')-Q_{t}^{\pi}(s,\pi)].$$

Thus we only need to bound the term $\sum_{t=1}^H\mathbb E_{s\sim d^{\pi'}_t}\mathbb I(Q_{t}^{\pi^*}(s,\pi')-Q_{t}^{\pi^*}(s,\pi)>p)[Q_{t}^{\pi}(s,\pi')-Q_{t}^{\pi}(s,\pi)]$.

According to Theorem~\ref{thm_pi}, 
\begin{align*}
    &\quad \ \sum_{t=1}^H\mathbb E_{s\sim d^{\pi'}_t}\mathbb I(Q_{t}^{\pi^*}(s,\pi')-Q_{t}^{\pi^*}(s,\pi)>p)[Q_{t}^{\pi}(s,\pi')-Q_{t}^{\pi}(s,\pi)]\\
    &\leq \sum_{t=1}^H\mathbb E_{s\sim d^{\pi'}_t}\mathbb I(Q_{t}^{\pi^*}(s,\pi')-Q_{t}^{\pi^*}(s,\pi)>p)\TV(\pi(s),\pi^*(s))(H-t+1)\\
    &\leq \delta\epsilon_b H^2
\end{align*}
Then we complete the proof.
\end{proof}

\section{Experiments on Approximated $Q^*$}\label{sec_extra}
In the theoretical analysis part, we use Q* for the bound deduction. Indeed, the exact expert Q-values are hard to get upfront. However, many existing methods can acquire a Q that is close to $Q^*$, including learning from offline datasets~\cite{cql, combo, hve}, using human advice~\cite{PEBBLE, human+}, and computing from rules~\cite{kernel, knowledge}. The cost depends on the choice of the surrogate of Q*.

To demonstrate the feasibility of a surrogate of $Q^*$, we experiment on the $Q^*$ obtained from an offline dataset. The experiment setting is the same as described in Sec. 7.1.1 in the main body. We choose CQL~\cite{cql} to learn a Q-network. The offline dataset is composed of only 5 expert trajectories (totaling 1351 transitions, which is about 20\% of samples used in the training process and about 7\% of samples used by DAgger). The result is shown in the following table.

\begin{table}[H]
\centering
\caption{Experiment on approximated $Q^*$.}\label{tab_approximate}
\begin{tabular}{l|l}
\toprule
Algorithm & mean (std)\\
\midrule
AdapMen ($Q^*$)  & 496.6 (13.0) \\
AdapMen (CQL) & 396.7 (17.2) \\
\addlinespace
DAgger &  366.2 (7.1) \\
BC &  376.2 (4.8) \\
CQL  & 21.6 (21.9) \\
\addlinespace
EnsembleDAgger  & 197.1 (3.7) \\
ValueDICE &  65.9 (5.1) \\
\bottomrule
\end{tabular}
\end{table}

The number in the parentheses is the standard deviation of five seeds. AdapMen ($Q^*$) uses the ground truth $Q^*$, while AdapMen (CQL) uses the $Q$ learned from CQL. The learned $Q$ suffers from little performance degradation, still outperforming baselines. CQL directly uses the learned $Q$ to update the policy, whose extremely low performance shows the necessity of AdapMen. 

\section{Extra Information of experiments in  Atari game}
\begin{table}[H]
    \centering
    \caption{Atari Expert Performance. Calculated from 100 trajectories.}
     \begin{tabular}[b]{|l|l|}
    \hline
    Task        &  mean (std)                \\ \hline
    MsPacman    & 1619 (2073)     \\ \hline
    BeamRider   & 1500 (1695)     \\ \hline
    DemonAttack & 230 (231)    \\ \hline
    Pong        & 5.72 (5.41)       \\ \hline
    Qbert       & 1205 (1934)     \\ \hline
    Enduro      & 180 (149)     \\ \hline
\end{tabular}
    \label{tab:atari performance}
\end{table}

\newpage
\section{Hyper-parameters}
\label{sec:parameter}

\begin{table}[h] 
\caption{Hyperparameters for implemented algorithms in Metadrive} 
\centering
\begin{tabular}{l|p{3.5cm}|c}
\toprule
\textbf{Type} & \textbf{Name} & \textbf{Value} \\
\toprule
\multirow{6}{*}{\textbf{General}}
        & learning rate & 1e-4                \\
        & policy net structure & [("mlp", 256), ("mlp", 256)]\\
        & batch size & 128                  \\
        & update learner intervel & 200     \\
        & number of batches per update& 50     \\
        & expert buffer size & 2000\\
\midrule
\multirow{1}{*}{\textbf{HG-DAgger} \& \textbf{EnsembleDAgger} } 
      & ensemble size & 5\\
\midrule
\multirow{3}{*}{\textbf{ValueDICE}} 
        & actor learning rate & 1e-5 \\
        & nu net structure & [("mlp", 256), ("mlp", 256)]\\
      & nu learning rate & 1e-3 \\
        & nu reg coeff & 10     \\
        & absorbing per episode & 10 \\
        & number of random actions & 1000 \\
\midrule
\multirow{1}{*}{\textbf{\methodname-Q} \& \textbf{\methodname-Pi}}  
      & horizon & 100\\
\bottomrule
\end{tabular}
\label{tab:hyperparameters1}
\end{table}

\newpage
\begin{table}[h] 
\caption{Hyperparameters for implemented algorithms in Atari} 
\centering
\begin{tabular}{l|p{3.5cm}|c}
\toprule
\textbf{Type} & \textbf{Name} & \textbf{Value} \\
\toprule
\multirow{6}{*}{\textbf{General}}
        & learning rate & 3e-4                \\
        & policy net hidden dim &  [("conv2d", 16, 8, 4, 0),\\&& ("conv2d", 32, 4, 2, 0),\\&&("flatten",),\\&& ("mlp", 256), ("mlp", 256)]\\
        & batch size & 32                  \\
        & update learner intervel & 200     \\
        & number of batches per update& 200     \\
        & expert buffer size & 50000\\
\midrule
\multirow{3}{*}{\textbf{ValueDICE}} 
        & actor learning rate & 1e-5 \\
        & nu net structure & [("conv2d", 16, 8, 4, 0), \\&&("conv2d", 32, 4, 2, 0),\\&&("flatten",), \\&&("mlp", 256), ("mlp", 128)]\\
      & nu learning rate & 1e-3 \\
        & nu reg coeff & 10     \\
        & number of random actions & 2000 \\
\midrule
\multirow{1}{*}{\textbf{\methodname-Q} \& \textbf{\methodname-Pi}}  
      & horizon & 100\\
\bottomrule
\end{tabular}
\label{tab:hyperparameters2}
\end{table}

\newpage
\section{Additional $D_Q$ distribution}\label{sec_dq}

\begin{figure}[htbp]
  \centering
  \subfigure[Enduro-initial]{
    \includegraphics[width=0.4\linewidth]{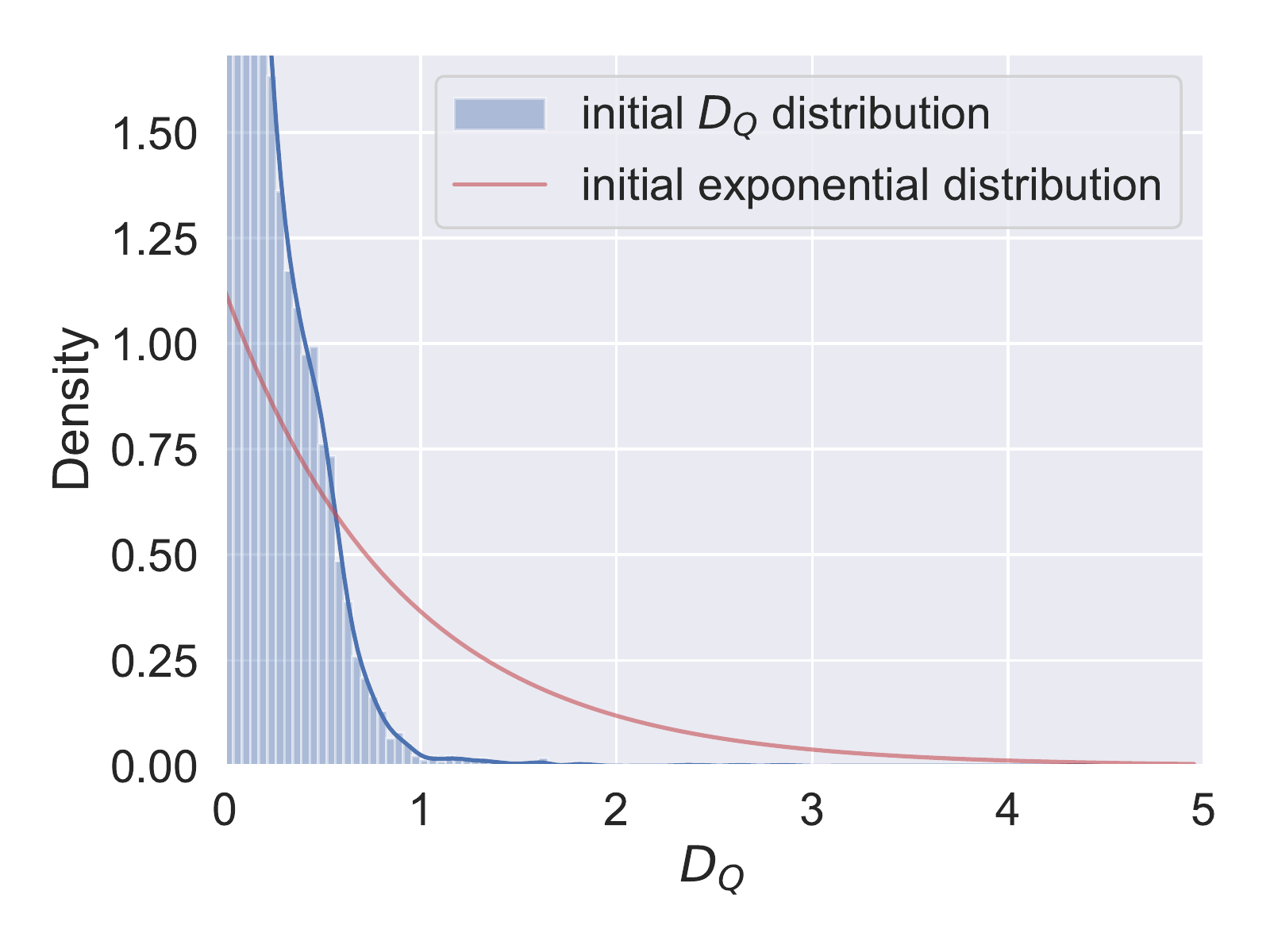}
  } 
  \subfigure[Enduro-final]{
    \includegraphics[width=0.4\linewidth]{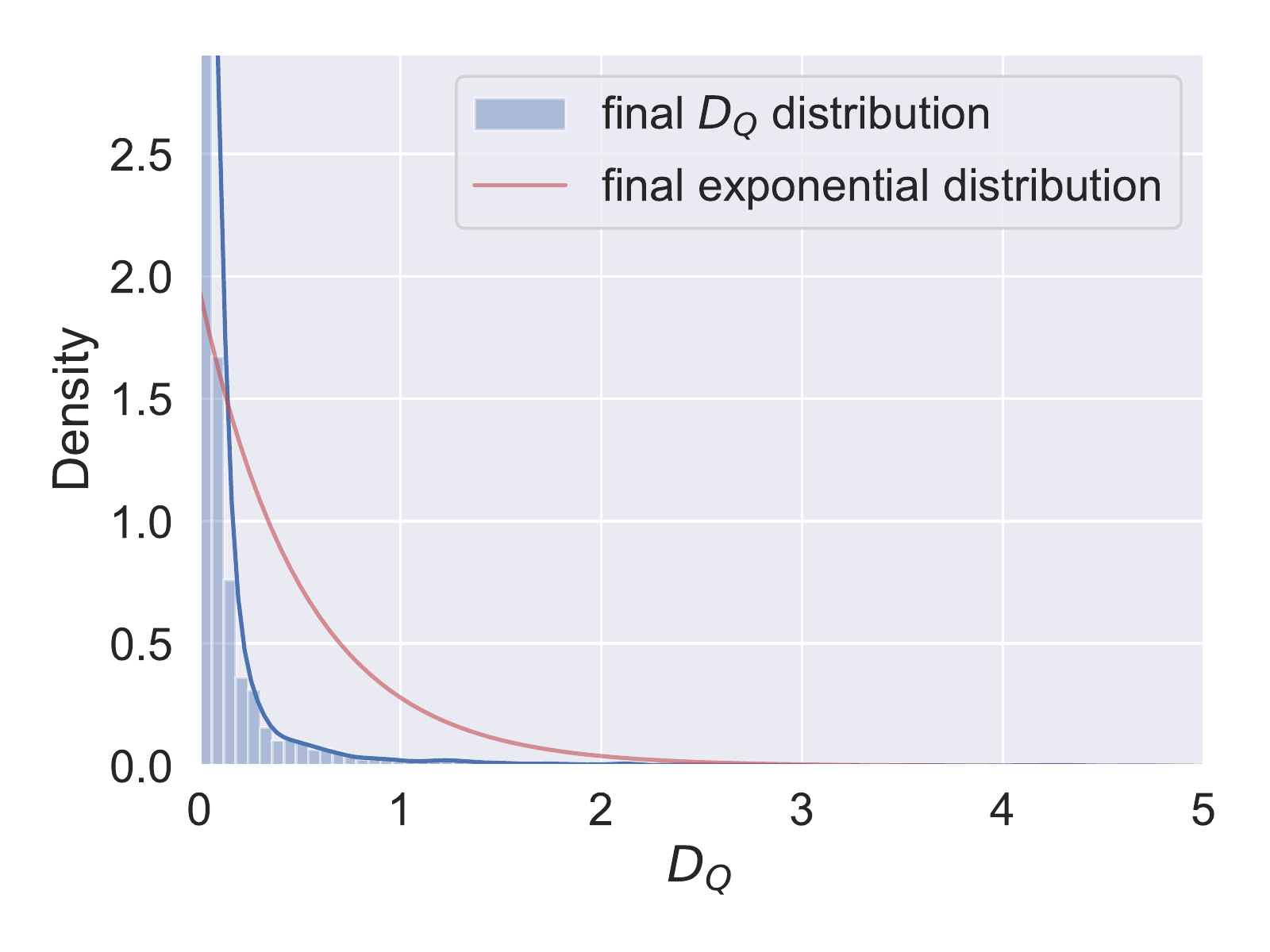}
  } 
  
  \subfigure[Qbert-initial]{
    \includegraphics[width=0.4\linewidth]{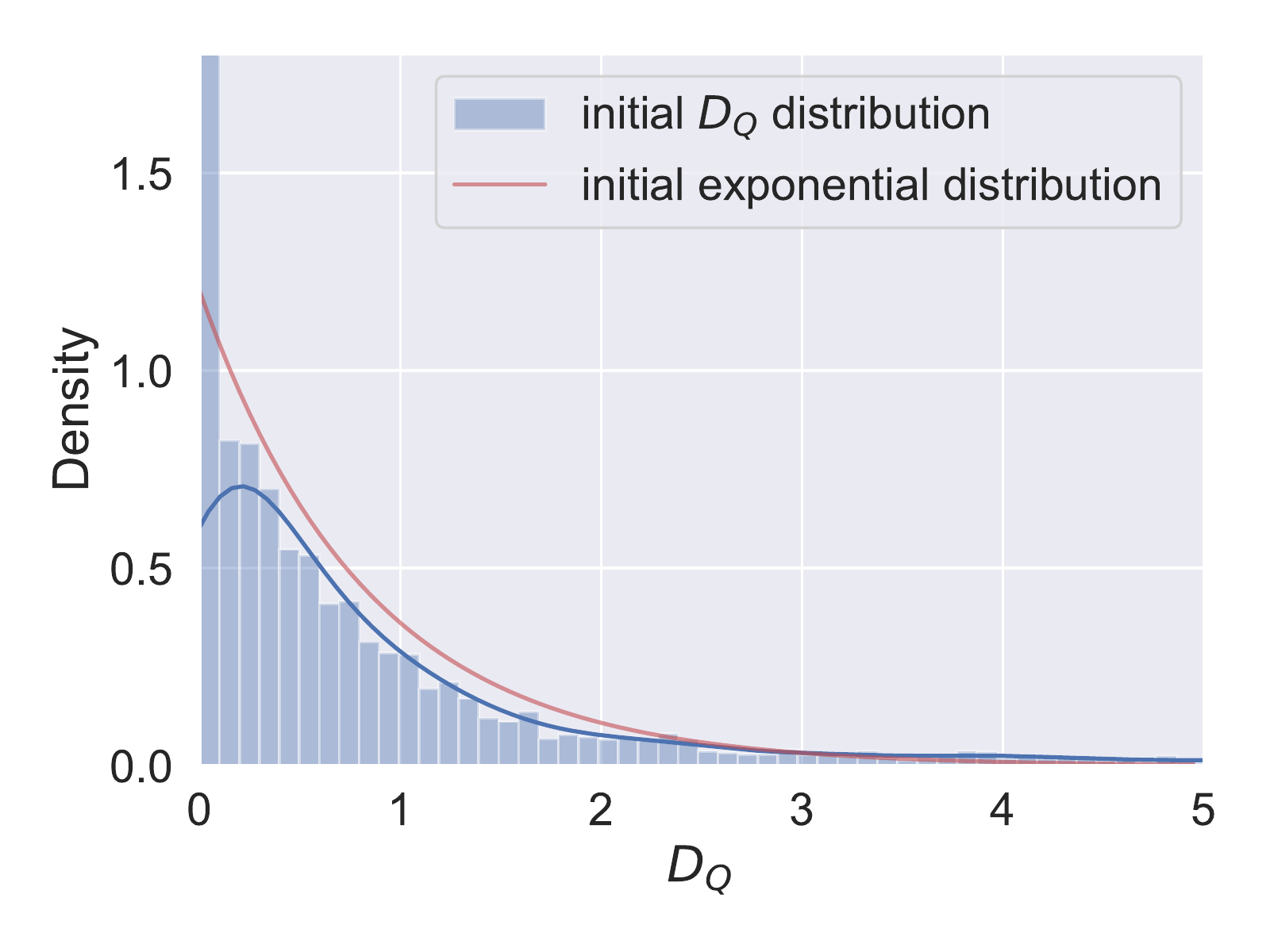}
  } 
  \subfigure[Qbert-final]{
    \includegraphics[width=0.4\linewidth]{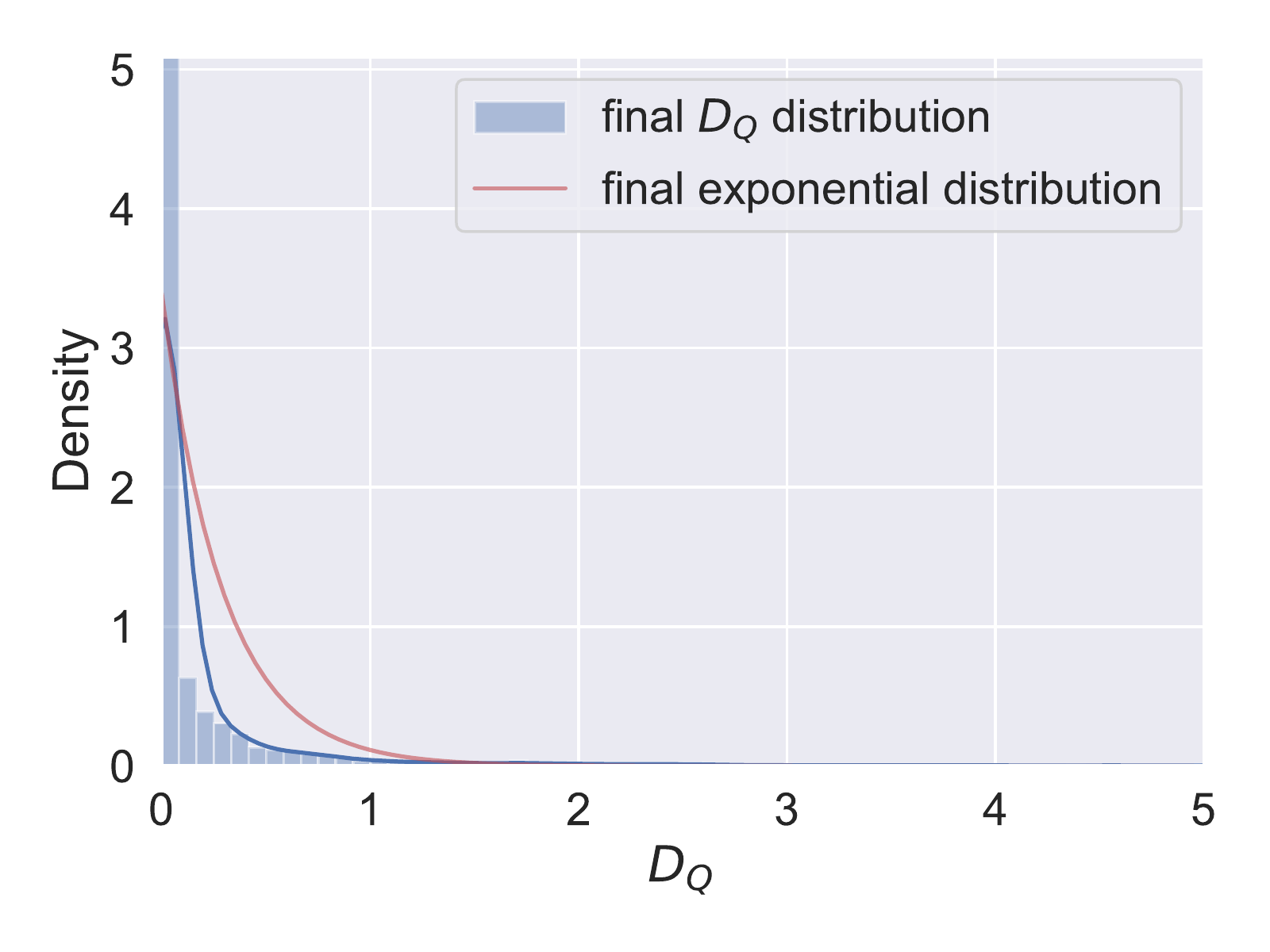}
    }
    
    \subfigure[Pong-initial]{
    \includegraphics[width=0.4\linewidth]{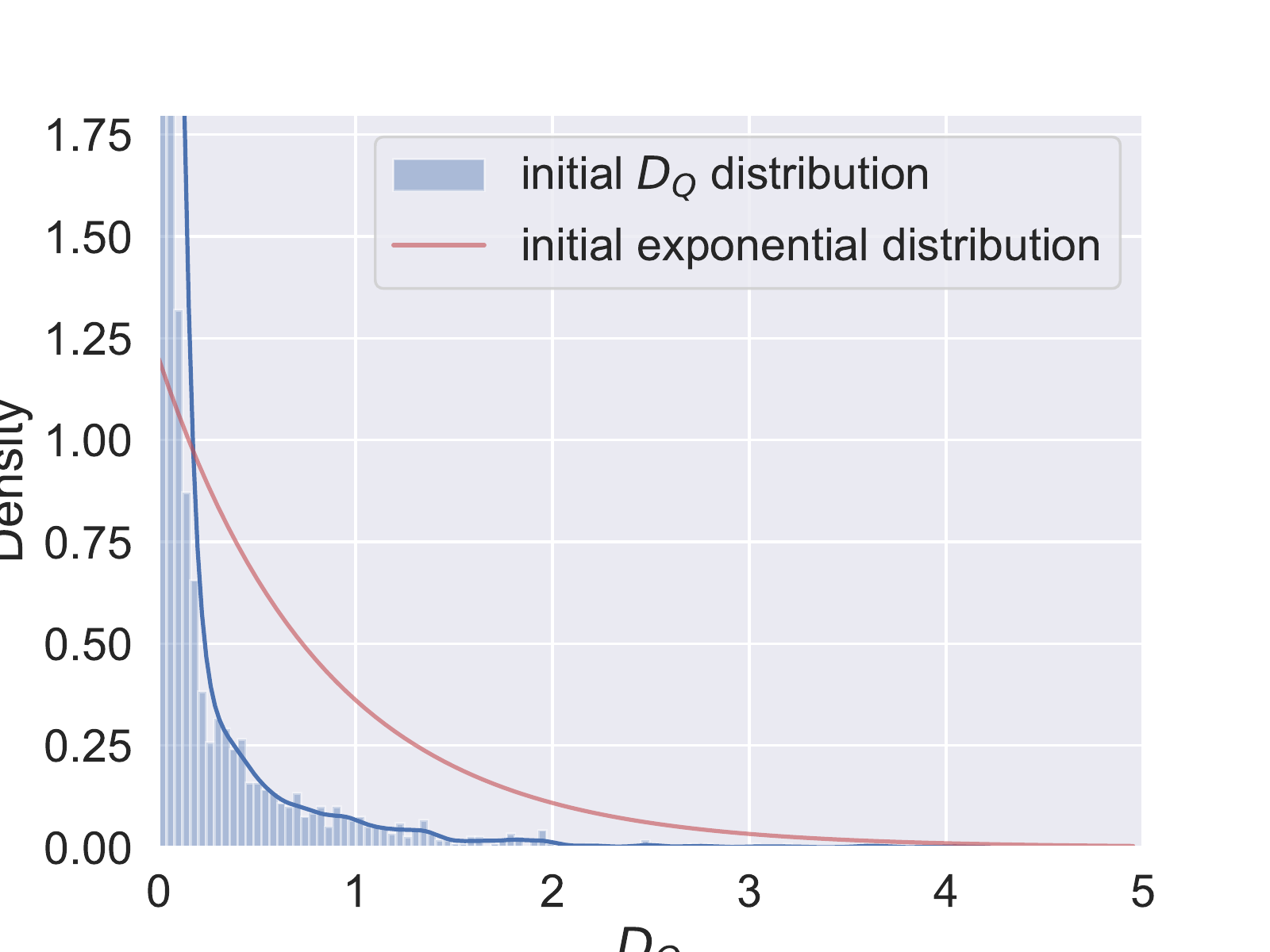}
  } 
  \subfigure[Pong-final]{
    \includegraphics[width=0.4\linewidth]{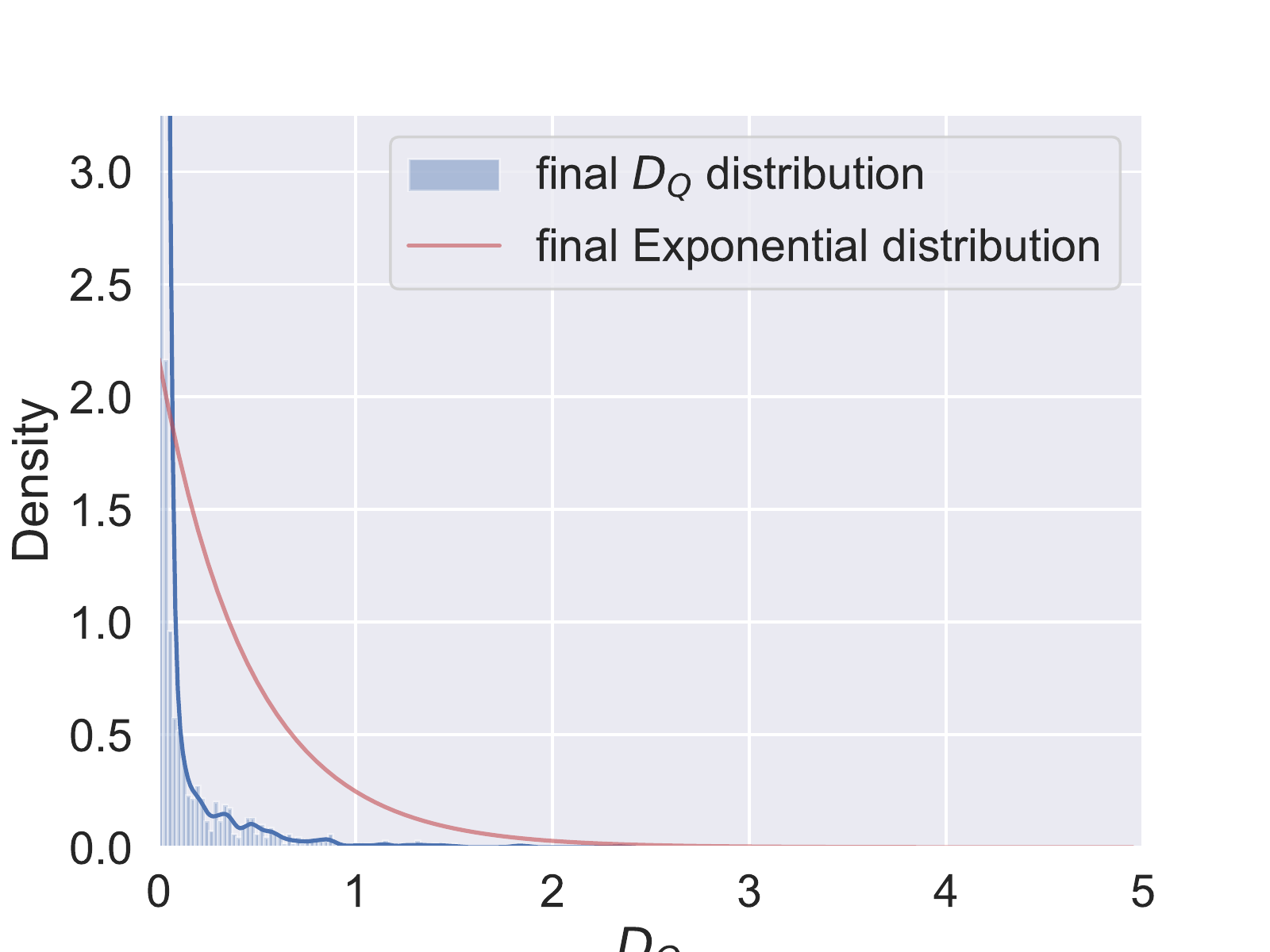}
    }
\caption{Additional $D_Q$ distributions in Atari games}
\end{figure}


\end{document}